\documentclass[9.5pt,journal,finalsubmission,compsoc]{IEEEtran}
\newcommand{\nop}[1]{}
\usepackage{wrapfig}
\usepackage[footnotesize]{subfigure}
\usepackage{chngpage}
\usepackage{flushend}
\usepackage{mathrsfs}
\usepackage{cite}
\usepackage{graphics}
\usepackage{graphicx}
\usepackage{epsfig}
\usepackage[scriptsize]{caption}
\usepackage{stfloats}
\usepackage{url}
\usepackage{epsf}
\usepackage[cmex10]{amsmath}
\usepackage{amsthm}
\usepackage{amssymb}
\usepackage{multirow}
\usepackage{paralist}
\usepackage{fullpage}
\usepackage{url}
\usepackage{algorithm,algorithmic}
\usepackage{defs}
\renewcommand{\algorithmicforall}{\textbf{On each}}

\begin{document}
\title{Scalable Kernel Clustering: Approximate Kernel \textit{k}-means}
\author{Radha~Chitta,
        Rong~Jin,
        Timothy~C.~Havens
        and~Anil~K.~Jain
\IEEEcompsocitemizethanks{\IEEEcompsocthanksitem R. Chitta, R. Jin, and A.K. Jain are with the Department
of Computer Science and Engineering, Michigan State University, East Lansing,
MI, 48824. T.C. Havens is with the Department of Electrical and Computer Engineering, Michigan Tech University, Houghton, MI, 49931.
E-mail: chittara@msu.edu, \{rongjin,jain\}@cse.msu.edu, thavens@mtu.edu
}
\thanks{A previous version of this paper appeared as~\cite{chitta2011approx}. In this version, we extend the proposed method to use ensemble clustering techniques and 
further enhance its performance. We also provide a tighter bound on the kernel approximation error when compared to that of the naive Nystrom 
approximation method. Empirical results on two additional large data sets along with more baseline techniques are presented to demonstrate the accuracy and scalability of the proposed algorithm.}}

\date{}

\IEEEcompsoctitleabstractindextext{
\begin{abstract}
Kernel-based clustering algorithms have the ability to capture the non-linear
structure in real world data. Among various kernel-based clustering algorithms,
kernel \textit{k}-means has gained popularity due to its simple iterative nature
and ease of
implementation. However, its run-time
complexity and memory footprint increase quadratically in terms of the size of
the data set, and hence, large data sets cannot be clustered efficiently. In
this
paper, we propose an approximation scheme based on randomization, called the
\emph{Approximate Kernel \textit{k}-means}. We approximate the cluster centers using the kernel similarity between a few sampled points and all the points in the data set. 
We show that the proposed method achieves better clustering performance than the traditional low rank kernel approximation based clustering schemes. We also demonstrate 
that it's running time and memory
requirements are significantly lower than
those of
kernel \textit{k}-means, with only a small reduction in the clustering quality on several public domain large data sets.  
We
then employ ensemble clustering techniques to further enhance the performance of
our algorithm.
\end{abstract}

}

\maketitle
\IEEEpeerreviewmaketitle
\IEEEdisplaynotcompsoctitleabstractindextext
\IEEEpeerreviewmaketitle
\section{Introduction}
Recent advances in data generation, collection and storage technologies have
resulted in a \emph{digital data explosion}.  A study by IDC and
EMC Corp\footnote{Refer \url{http://idcdocserv.com/1142}} predicted the creation of 8 trillion gigabytes of
digital data by the year 2015. Massive amounts
of data are generated
through online services like blogs, e-mails and social networks in the form of
text, images, audio and video. Clustering is one
of the principal tools to efficiently organize such 
large amounts of data and to enable convenient access. It has
found use in a multitude of applications such as web search, social network
analysis, image retrieval, medical imaging, gene
expression analysis, recommendation systems and market
analysis~\cite{jain2010data}.

Most algorithms capable of clustering large data sets 
assume that the clusters in the data set are linearly separable and group the
objects based on their pairwise Euclidean distances. On the other hand,
kernel-based clustering algorithms employ a non-linear distance measure, defined
in terms of a positive-definite kernel, to compute the similarity. The kernel
function embeds the objects in a
high-dimensional feature space, in which the clusters are more likely to be
separable.
Kernel clustering algorithms, therefore, have the ability to
capture the non-linear structure in real world data sets and, thus, usually perform
better than the
Euclidean distance based clustering algorithms~\cite{kim2005evaluation}.

A number of kernel-based clustering
methods such as spectral clustering~\cite{von2007tutorial},
kernel Self-Organizing Maps (SOM)~\cite{macdonald2002kernel} and kernel
neural gas~\cite{qin2004kng} have been proposed. In this study, we focus on
kernel \textit{k}-means~\cite{girolami2002mercer,scholkopf1996nonlinear} due to
its
simplicity and efficiency. 
In addition, several studies have
established the equivalence of kernel \textit{k}-means and other kernel-based
clustering methods, suggesting that they yield
similar results~\cite{dhillon2004unified,ding2005equivalence,zha2001spectral}.

Kernel \textit{k}-means~\cite{girolami2002mercer}, a non-linear extension of the classical
\textit{k}-means algorithm, replaces the Euclidean distance function
$d^{2}(x_a,x_b) = \| x_a - x_b\|^2$ employed in the \textit{k}-means algorithm
with a non-linear kernel distance function defined as
\[
 d_{\kappa}^{2}(x_a,x_b) = \kappa(x_a,x_a) + \kappa(x_b,x_b) - 2\kappa(x_a,x_b),
\]
where $x_a \in \Re^d$ and $x_b \in \Re^d$ are two data points and 
$\kappa(.,.) : \Re^d \times \Re^d \rightarrow \Re $ is the kernel function.
While the kernel distance function enables the clustering algorithm to capture
the non-linear structure in data, it requires computation and storage of an $n
\times n$ kernel matrix in memory, where $n$ is the number of data points to be clustered. 
This renders kernel \textit{k}-means non-scalable to data sets with
more than a few thousands of data points, on a standard workstation.
In this paper, we address the challenge posed by the large kernel matrix.

The Nystrom method for kernel approximation has been successfully employed in several learning problems~\cite{belabbas2009spectral,drineas2005nystrom,fowlkes2004spectral,williams2001using}. 
The naive Nystrom approximation method~\cite{gittens2011spectral} randomly samples a small number of points from the data set and 
computes a low rank approximation of the kernel matrix using the similarity between all the points and the sampled points.  
The spectral clustering algorithm was adapted to use this low rank approximate kernel~\cite{5444877,fowlkes2004spectral,liu2010segmentation,Yan:EECS-2009-45}. 
The clusters are found by approximating the top eigenvectors of the kernel using the similarity between a subset of randomly selected data points.  
The proposed algorithm, named \emph{Approximate kernel \textit{k}-means (aKKm)}, follows along the idea of the Nystrom approximation and avoids computing the full $n \times n$ kernel
matrix. We randomly select a subset of $m$ data points ($m \ll n$), and approximate the
cluster centers using vectors in the subspace spanned by this subset. This
approximation requires the computation and storage of only the $n \times m$  portion of
the kernel matrix, leading to a significant speedup of kernel
\textit{k}-means. \textbf{We demonstrate, both theoretically and empirically, that aKKm yields similar
clustering performance as kernel \textit{k}-means using the full kernel matrix.} Unlike the spectral clustering algorithm based on the naive Nystrom extension~\cite{fowlkes2004spectral}, 
our method 
uses information from all the eigenvectors of the approximate kernel matrix (without explicitly computing them), thereby yielding more accurate clustering results. 
We further improve the
efficacy of our algorithm 
through ensemble clustering methods.

\section{Background}
\label{sec:background}
We first briefly describe some of the related work on large scale
clustering and kernel-based clustering, and then outline
the kernel \textit{k}-means algorithm. 

\subsection{Large scale clustering} A number of methods have been developed to 
efficiently cluster large data sets. Incremental
clustering~\cite{can1993incremental,can1995incremental} and divide-and-conquer
based clustering algorithms~\cite{aggarwal2003framework,guha2003clustering} were
designed to operate in a single pass over the data points, thereby reducing the
time required for clustering. Sampling based methods, such as
CLARA~\cite{kaufmanfinding} and CURE~\cite{guha2001cure}, reduce the computation
time by finding the cluster centers based on a small number of randomly selected
data points. The coreset algorithms~\cite{har2004coresets} represent the
data set using a small set of \emph{core} data points and find the cluster
centers using only these core data points. Clustering algorithms such as
BIRCH~\cite{zhang1996birch} and CLARANS~\cite{ng2002clarans} improve the
clustering efficiency by summarizing the data set into data structures like
trees and graphs, thus enabling efficient data access. 

With the evolution of cloud computing, parallel processing techniques for 
clustering are gaining popularity~\cite{das2007google,ene2011fast}. These
techniques
speedup the clustering process by first dividing the task into a number of
independent sub-tasks that can be performed simultaneously, and then efficiently
merging these solutions into the final solution. For instance, in
\cite{ene2011fast}, the 
MapReduce framework~\cite{ranger2007evaluating} is employed to speedup 
the \textit{k}-means and the \textit{k}-medians clustering algorithms. The data
set is
split among many processors 
and a small representative data sample is obtained from each of the processors.  
These representative data points are then clustered to obtain the cluster
centers or medians.

\subsection{Kernel-based clustering} Most of the existing methods for large 
scale clustering compute the pairwise dissimilarities between the data points
using the Euclidean distance measure. As a result, 
they cannot accurately cluster data sets that are not linearly separable. Kernel
based
clustering techniques address this limitation by employing a non-linear kernel
distance
function to capture the non-linear structure in data~\cite{kim2005evaluation}.
Various kernel-based clustering algorithms have been developed, including kernel
\textit{k}-means~\cite{scholkopf1996nonlinear,girolami2002mercer}, spectral
clustering~\cite{von2007tutorial}, kernel SOM~\cite{macdonald2002kernel} and
kernel neural
gas~\cite{qin2004kng}.

Scalability is a major challenge faced by all the kernel-based algorithms, as
they 
require computation of the full kernel matrix whose size is quadratic in the
number
of data points. To the best of our knowledge,
only a few attempts have been
made to scale kernel clustering algorithms to large data sets. In~\cite{zhang2002large}, the memory requirement is reduced by dividing the
kernel
matrix into blocks and using one block of the kernel matrix at a time. Although
this technique handles the memory complexity, it still requires the computation
of the full kernel matrix. The leaders clustering algorithm is integrated with kernel \textit{k}-means to reduce its computational complexity in~\cite{sarma2012speeding}. 
However, this method is data order-dependent and does not always produce accurate results. 
Sampling methods, such as the Nystrom
method~\cite{williams2001using}, have been employed to obtain low rank
approximation of the kernel matrix to
address this challenge~\cite{fowlkes2004spectral, liu2010segmentation}. 
In~\cite{sakai2009fast}, 
random projection is combined with sampling to further improve the clustering
efficiency. However, these methods rely on the approximation of the top eigenvectors of the kernel matrix. 
In our method, we propose to use the approximate kernel matrix directly to find the clusters and show that this leads to more accurate clustering of the data. 
\subsection{Kernel \textit{k}-means}
Let $X = \{x_1,x_2,...,x_n\}$ be the input data set consisting of $n$ data 
points, where $x_i \in \Re^d$, $C$ be the number of clusters and $K \in
\Re^{n\times n}$ be the kernel matrix with $K_{ij} = \kappa(x_i, x_j)$, where
$\kappa(\cdot, \cdot)$ is the kernel function. Let $\Hk$ be the Reproducing
Kernel Hilbert Space (RKHS) endowed by the kernel function $\kappa(\cdot,
\cdot)$, and $|\cdot|_{\Hk}$ be the functional norm for $\Hk$. The objective of
kernel \textit{k}-means is to minimize the \textit{clustering error}, defined as
the sum of squared distances between the data points and the center of the
cluster to which each point is assigned. Hence, the kernel \textit{k}-means
problem can be cast as the following optimization problem~\cite{zha2001spectral}:
\begin{eqnarray}
\min\limits_{U  \in \P} \max\limits_{\{c_k(\cdot) \in \Hk\}_{k=1}^C}
\sum_{k=1}^C \sum_{i=1}^n U _{ki} |c_k(\cdot) - \kappa(x_i, \cdot)|_{\Hk}^2,
\label{eqn:kk-1}
\end{eqnarray}
where $U  = (\u_1, \ldots, \u_C)^{\top}$ is the cluster membership matrix, 
$c_k(\cdot) \in \Hk, k \in [C]$ are the cluster centers, and domain $\P = \{ U
\in \{0, 1\}^{C\times n}: U^{\top}\mathbf{1} = \mathbf{1} \}$, where
$\mathbf{1}$ is a vector of all ones. Let $n_k = \u_k^{\top}\mathbf{1}$ be
the number of data points assigned to the $k^{th}$ cluster, and 
\begin{eqnarray*}
\Uh& = & (\uh_1, \ldots, \uh_C)^{\top} = [\diag(n_1, \ldots, n_C)]^{-1} U , \\
\Ut & = & (\ut_1, \ldots, \ut_C)^{\top} = [\diag(\sqrt{n_1}, \ldots,
\sqrt{n_C})]^{-1}U ,
\end{eqnarray*}
denote the $\ell_1$ and $\ell_2$ normalized membership matrices, respectively.

The problem in \eqref{eqn:kk-1} can be relaxed to the following optimization problem over
$U$~\cite{zha2001spectral}:
\begin{eqnarray}
\min\limits_{U } \tr(K) - \tr(\Ut K \Ut^{\top}), \label{eqn:kk-1-a}
\end{eqnarray}
and the optimal cluster centers found using 
\begin{eqnarray}
c_k(\cdot) = \sum_{i=1}^n \Uh_{ki} \kappa(x_i, \cdot), k \in [C].
\label{eqn:c-1}
\end{eqnarray}
As indicated in \eqref{eqn:kk-1-a}, a naive implementation of kernel
\textit{k}-means requires computation and storage of the full $n \times n$
kernel matrix $K$, restricting its scalability. The objective of our work is to reduce the
computational complexity and the memory requirements of kernel \textit{k}-means.
\section{Approximate Kernel \textit{k}-means} 
\label{sec:approxkkmeans}
A simple and naive approach for reducing the complexity of kernel 
\textit{k}-means is to randomly sample $m$ points from the data set to be
clustered, and find the cluster centers based only on the sampled
points; then assign every unsampled data point to the cluster whose center is
nearest. We refer to this two-step process as the \emph{two-step kernel
\textit{k}-means} (tKKm), detailed in Algorithm~\ref{alg:2step_kkmeans}. Though this
approach has reduced run-time complexity and memory requirements, its
performance does not match that of the kernel \textit{k}-means algorithm,
unless it
is provided with a sufficiently large sample of data points.

We propose a superior approach for reducing the complexity of kernel 
\textit{k}-means based on the fact that kernel \textit{k}-means requires the full $n \times n$ kernel matrix $K$ only because the the cluster centers 
$\{c_k(\cdot), k \in [C]\}$ are represented as linear combinations of
\textit{all} the data points to be clustered (see \eqref{eqn:c-1})~\cite{scholkopf2001generalized}. In other words, the cluster centers lie in the space spanned
by \textit{all} the data points, i.e., $c_k(\cdot) \in \H_a = \spana(\kappa(x_1,
\cdot), \ldots, \kappa(x_n, \cdot)), k \in [C]$. We can avoid computing the full kernel
matrix if we restrict the cluster centers to a smaller subspace
$\H_b \subset \H_a$. $\H_b$ should be constructed such that (i) $\H_b$ is small
enough to allow efficient computation, and (ii) $\H_b$ is rich enough to yield
similar clustering results as those obtained using $\H_a$. We employ a simple 
approach of randomly sampling $m$ data points ($m \ll
n$), denoted by $\Xh = \{\xh_1, \ldots, \xh_m \}$, and construct the subspace
$\H_b = \spana(\xh_1, \ldots, \xh_m)$. Given the subspace $\H_b$, we modify
\eqref{eqn:kk-1} as
\begin{eqnarray}
\min\limits_{U  \in \P} \max\limits_{\{c_k(\cdot) \in \H_b\}_{k=1}^C} 
\sum_{k=1}^C \sum_{i=1}^n U _{ki} |c_k(\cdot) - \kappa(x_i, \cdot)|_{\Hk}^2.
\label{eqn:kk-2}
\end{eqnarray}
Let $K_B \in \Re^{n\times m}$ represent the kernel similarity matrix between
data points in $X$ and the sampled data points in $\Xh$, and $\Kh \in \Re^{m\times
m}$ represent the kernel similarity between the sampled data points. The
following lemma allows us to reduce \eqref{eqn:kk-2} to an optimization problem
involving only the cluster membership matrix $U$.
\begin{lemma} \label{lemma:1}
Given the cluster membership matrix $U $, the optimal cluster centers in
\eqref{eqn:kk-2} are given by
\begin{eqnarray}
c_k(\cdot) = \sum_{i=1}^m \alpha_{ki} \kappa(\xh_i, \cdot),
\end{eqnarray}
where $\alpha = \Uh K_B\Kh^{-1}$. The optimization problem for $U $ is given by
\begin{eqnarray}
\min\limits_{U } \tr(K) - \tr(\Ut K_B \Kh^{-1} K_B^{\top} \Ut^{\top}).
\label{eqn:kk-3}
\end{eqnarray}
\end{lemma}
\begin{proof}
 Let $\varphi_i = (\kappa(x_i, \xh_1), \ldots, \kappa(x_i, \xh_m))$ and 
\newline $\alpha_i  =  (\alpha_{i1}, \ldots , \alpha_{im}) $  be the $i$-th rows
of matrices $K_B$ and $\alpha$ respectively. As $c_k(\cdot) \in \H_b =
\spana(\xh_1, \ldots, \xh_m)$, we can write $c_k(\cdot)$ as
\[
    c_k(\cdot) = \sum_{i=1}^m \alpha_{ki} \kappa(\xh_i, \cdot).
\]
and write the objective function in (9) as
\begin{eqnarray}
& & \sum_{k=1}^C \sum_{i=1}^n U _{ki}|c_k(\cdot) - \kappa(x_i, \cdot)|_{\Hk}^2
\nonumber \\
& = & \tr(K) + \sum_{k=1}^C\left( n_k\alpha_k^{\top} \Kh\alpha_k - 2
\u_k^{\top}K_B\alpha_k\right). \label{eqn:temp-1}
\end{eqnarray}
By minimizing over $\alpha_k$, we have
\[
    \alpha_k = \Kh^{-1}K_B^{\top} \uh_k, k \in [C]
\]
and therefore, $\alpha = \Uh K_B\Kh^{-1}$. We complete the proof by substituting
the expression for $\alpha$ into \eqref{eqn:temp-1}.
\end{proof}

As indicated by Lemma~\ref{lemma:1}, we need to compute only
$K_B$ for finding the cluster memberships\footnote{$\Kh$ is part of $K_B$ and
therefore 
does not need to be computed separately.}. When $m \ll n$,
this computational cost is significantly smaller than that of computing
the full kernel matrix. 
On the other hand, when $m = n$, i.e., all the data points are selected for 
constructing the subspace $\H_b$, we have $\Kh = K_B = K$ and the problem
in \eqref{eqn:kk-3} reduces to \eqref{eqn:kk-1-a}.  We refer to
the proposed algorithm as \emph{Approximate Kernel \textit{k}-means} (aKKm), outlined
in Algorithm~\ref{alg:approx_kkmeans}.
Fig.~\ref{fig:synth_data} illustrates and compares this algorithm with tKKm 
on a $2$-dimensional synthetic data set. Note that the problem in \eqref{eqn:kk-3} 
can also be viewed as approximating 
the kernel matrix $K$ in \eqref{eqn:kk-2} by $K_B \Kh^{-1} K_B^{\top}$, which is
essentially the Nystrom method for low rank matrix approximation. 
However, our method offers two advantages over previous learning methods which employ 
the Nystrom approximation. Firstly, we do not need to explicitly compute the top eigenvectors of the approximate kernel matrix, 
resulting in a higher speedup over kernel \textit{k}-means. Secondly, our method uses the approximate kernel matrix directly to estimate the clusters instead 
of the top eigenvectors. We demonstrate through our analysis that this leads to a more accurate solution than that obtained by the earlier methods.
\begin{figure*}
\centering
\begin{tabular}{cc}
\subfigure[]{\label{fig:synth1}\includegraphics[width=3cm,height=2.3cm]{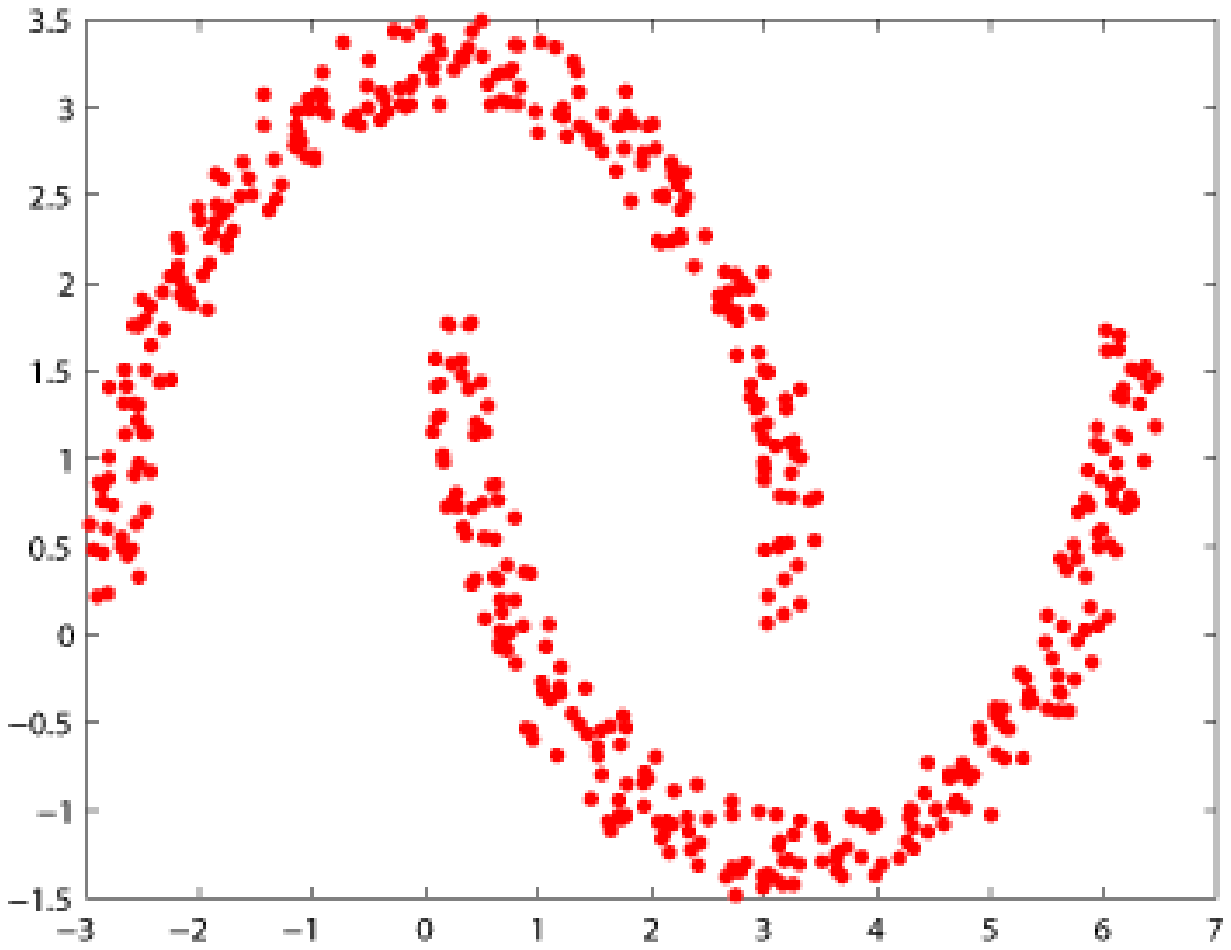}} &
\subfigure[]{\label{fig:synth2}\includegraphics[width=3cm,height=2.3cm]{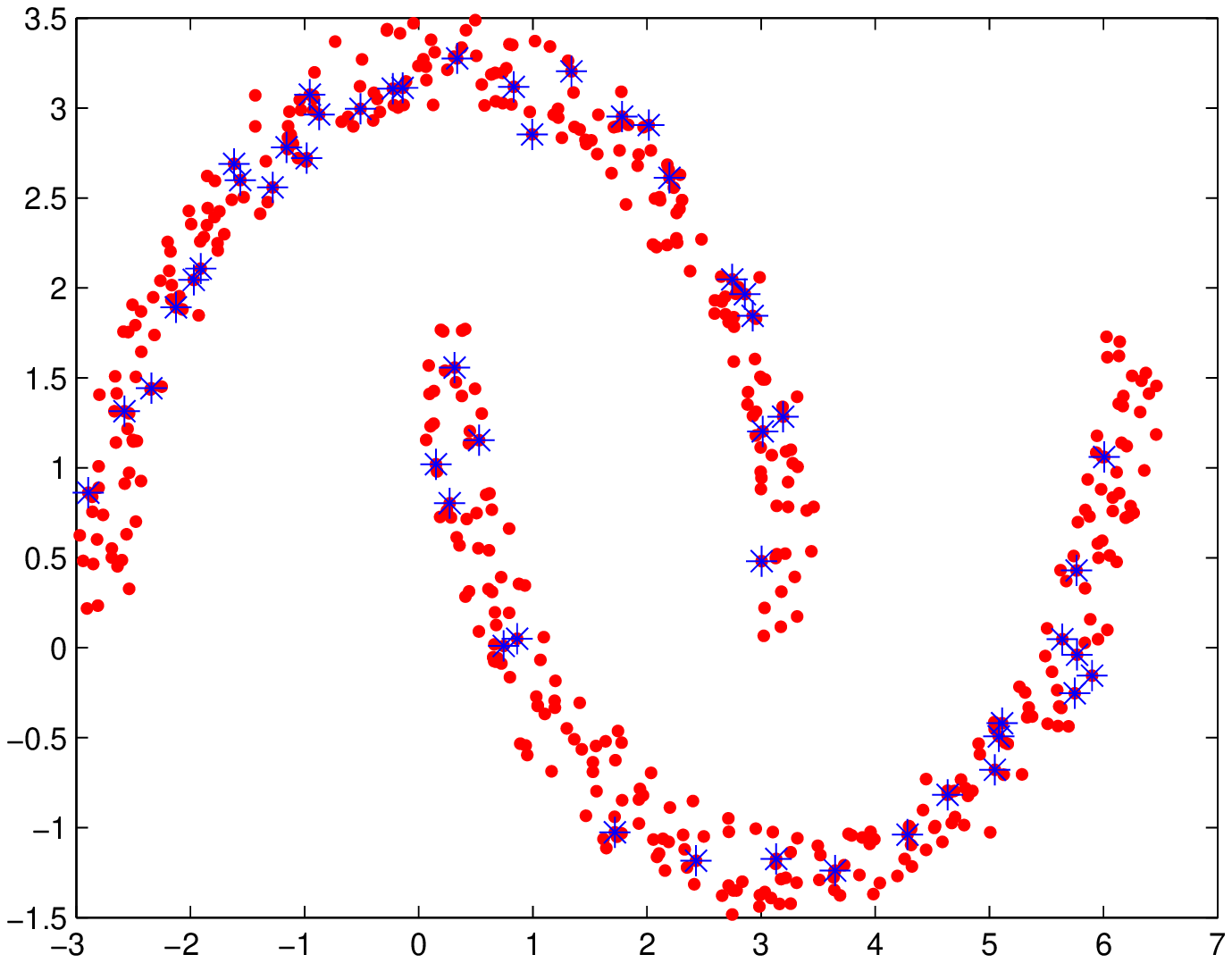}}
\end{tabular}
\begin{tabular}{ccc}
\subfigure[]{\label{fig:synth3}\includegraphics[width=3cm,height=2.3cm]{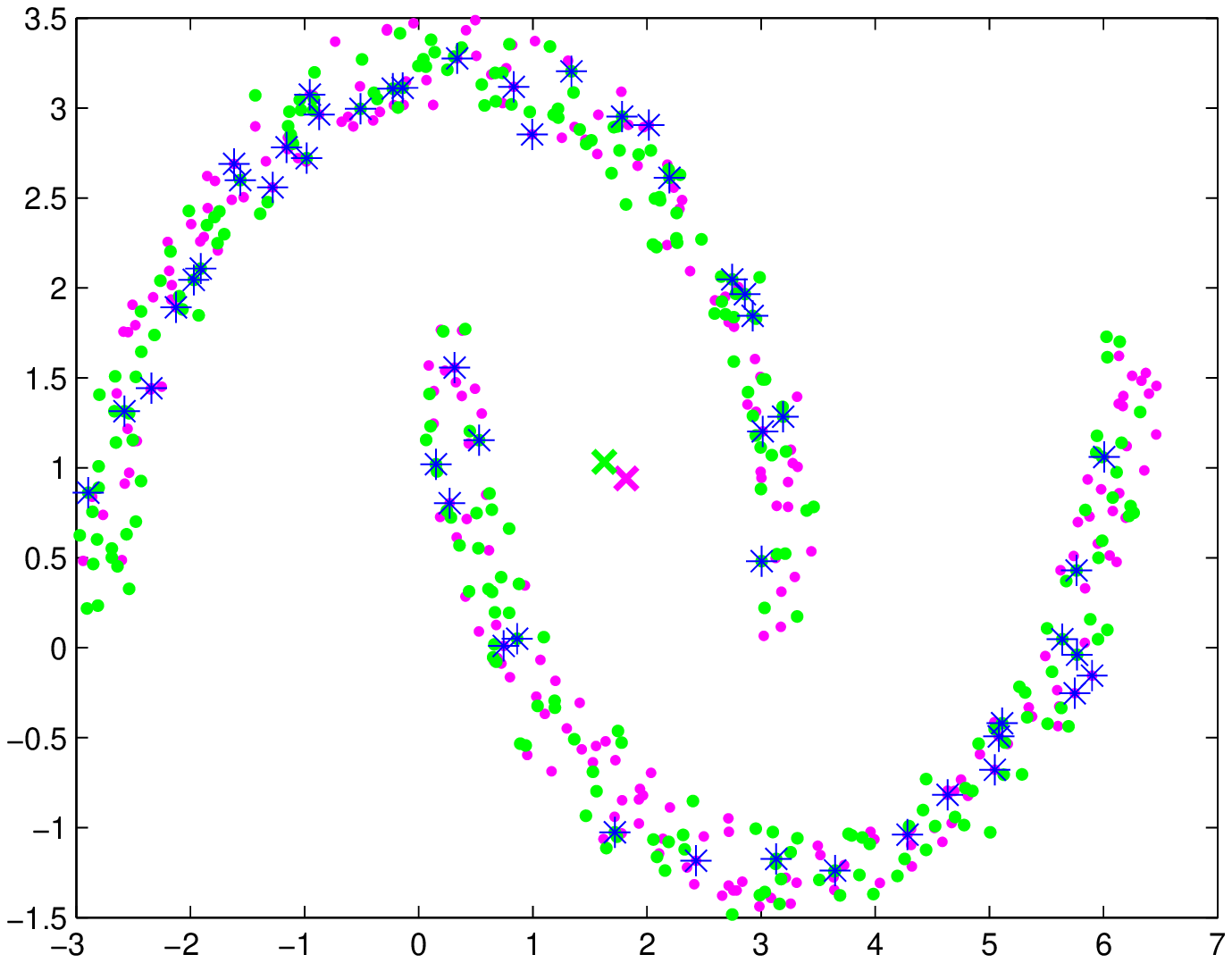}} & 
\subfigure[]{\label{fig:synth4}\includegraphics[width=3cm,height=2.3cm]{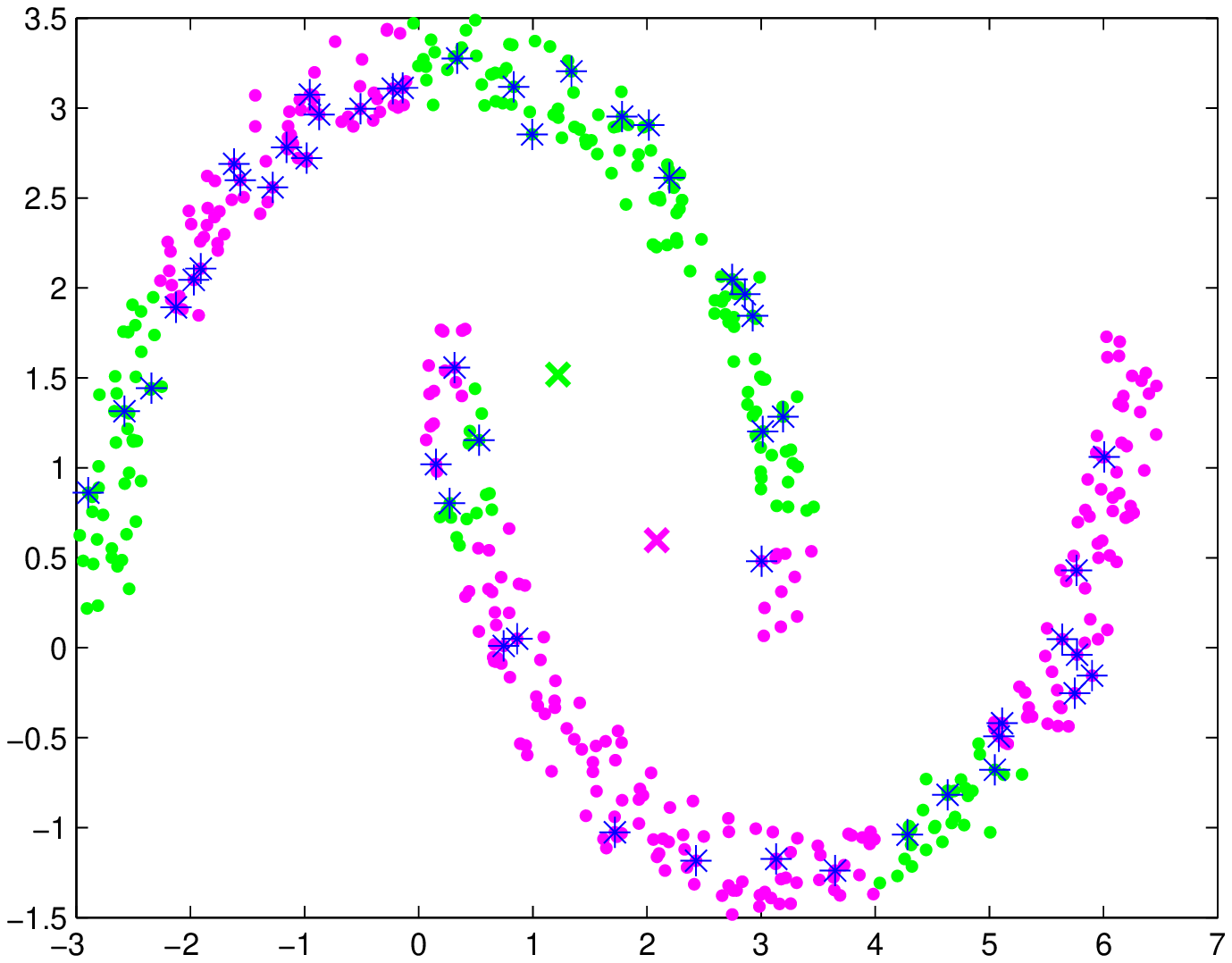}} &
\subfigure[]{\label{fig:synth5}\includegraphics[width=3cm,height=2.3cm]{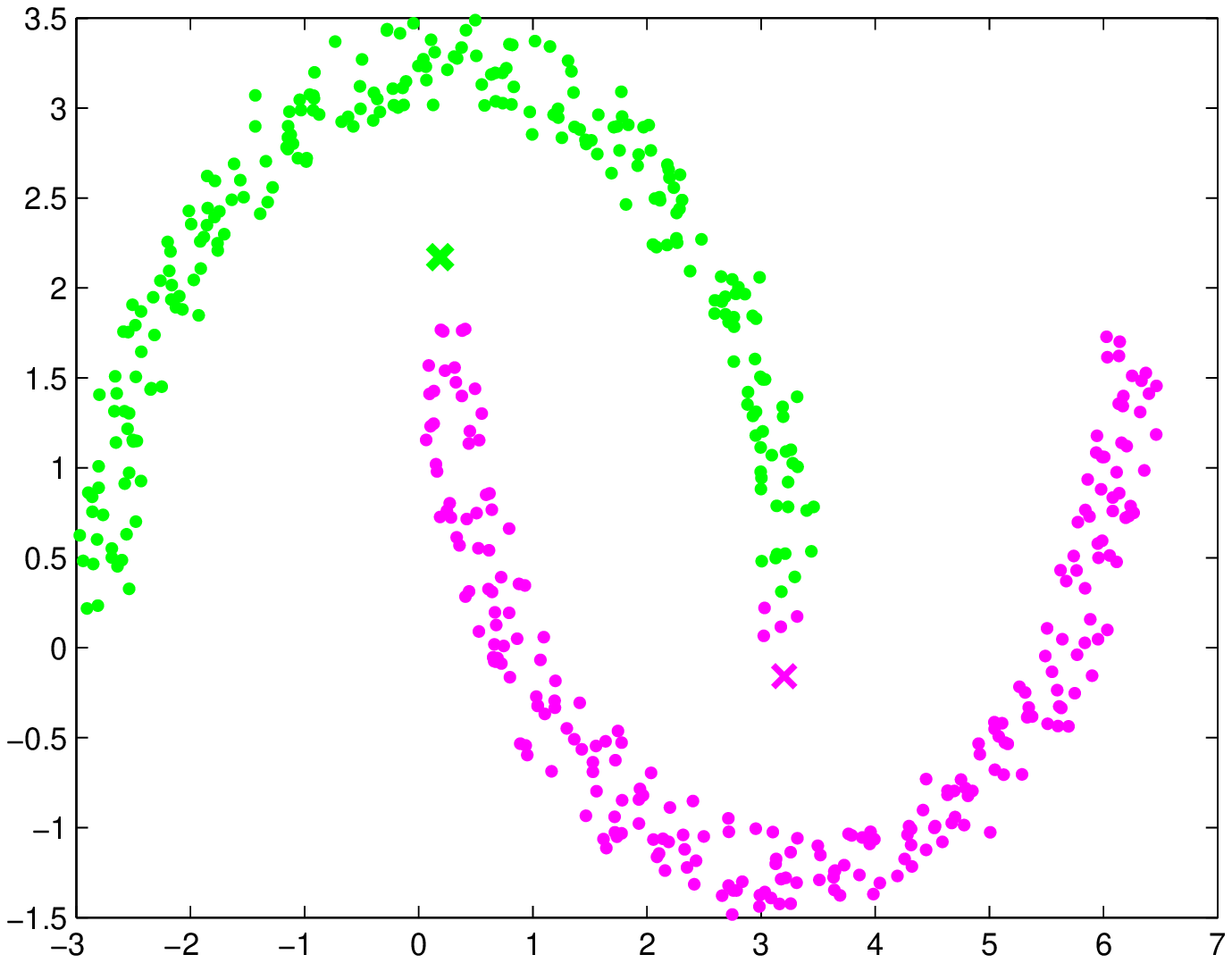}} \\
\subfigure[]{\label{fig:synth6}\includegraphics[width=3cm,height=2.3cm]{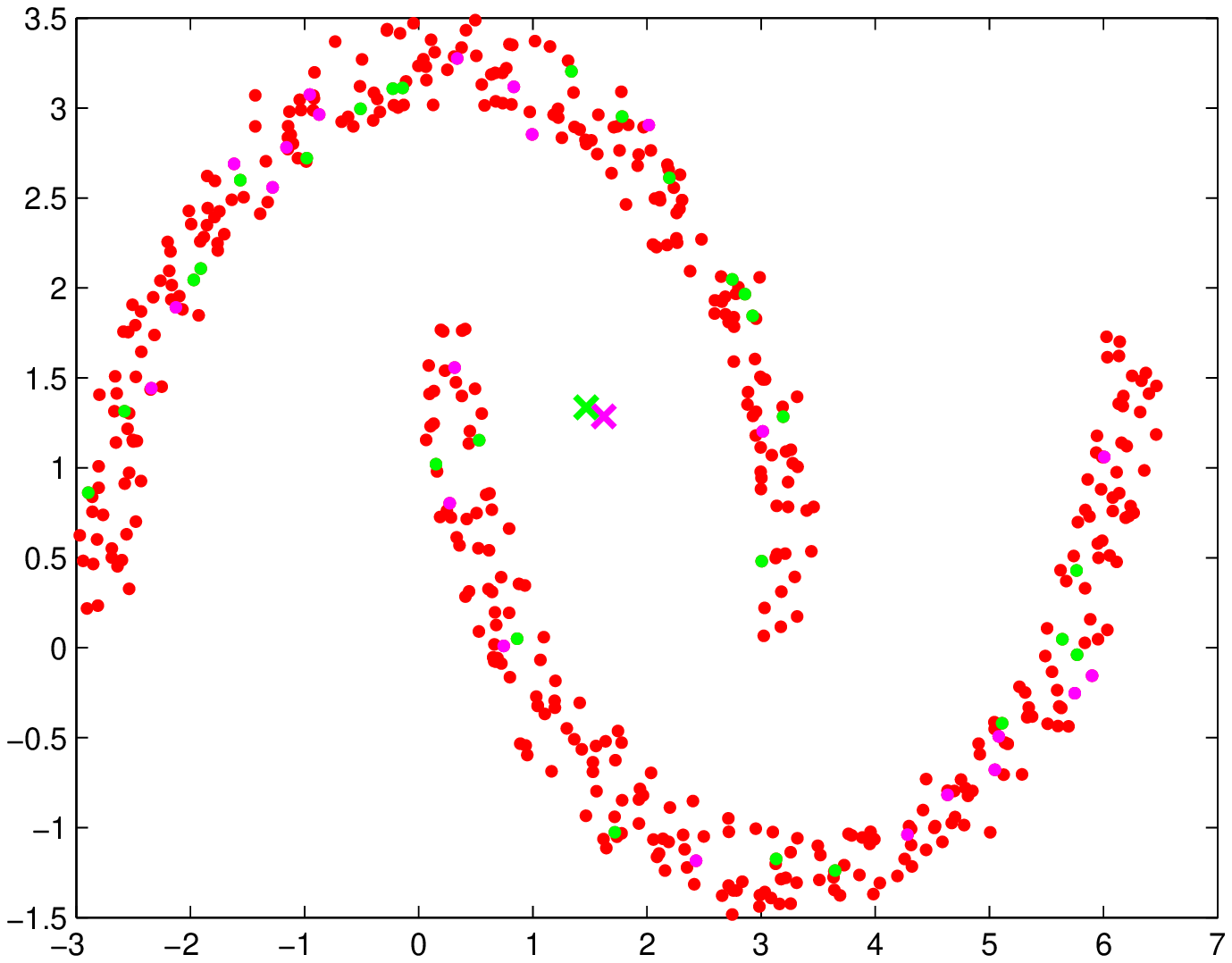}} &
\subfigure[]{\label{fig:synth7}\includegraphics[width=3cm,height=2.3cm]{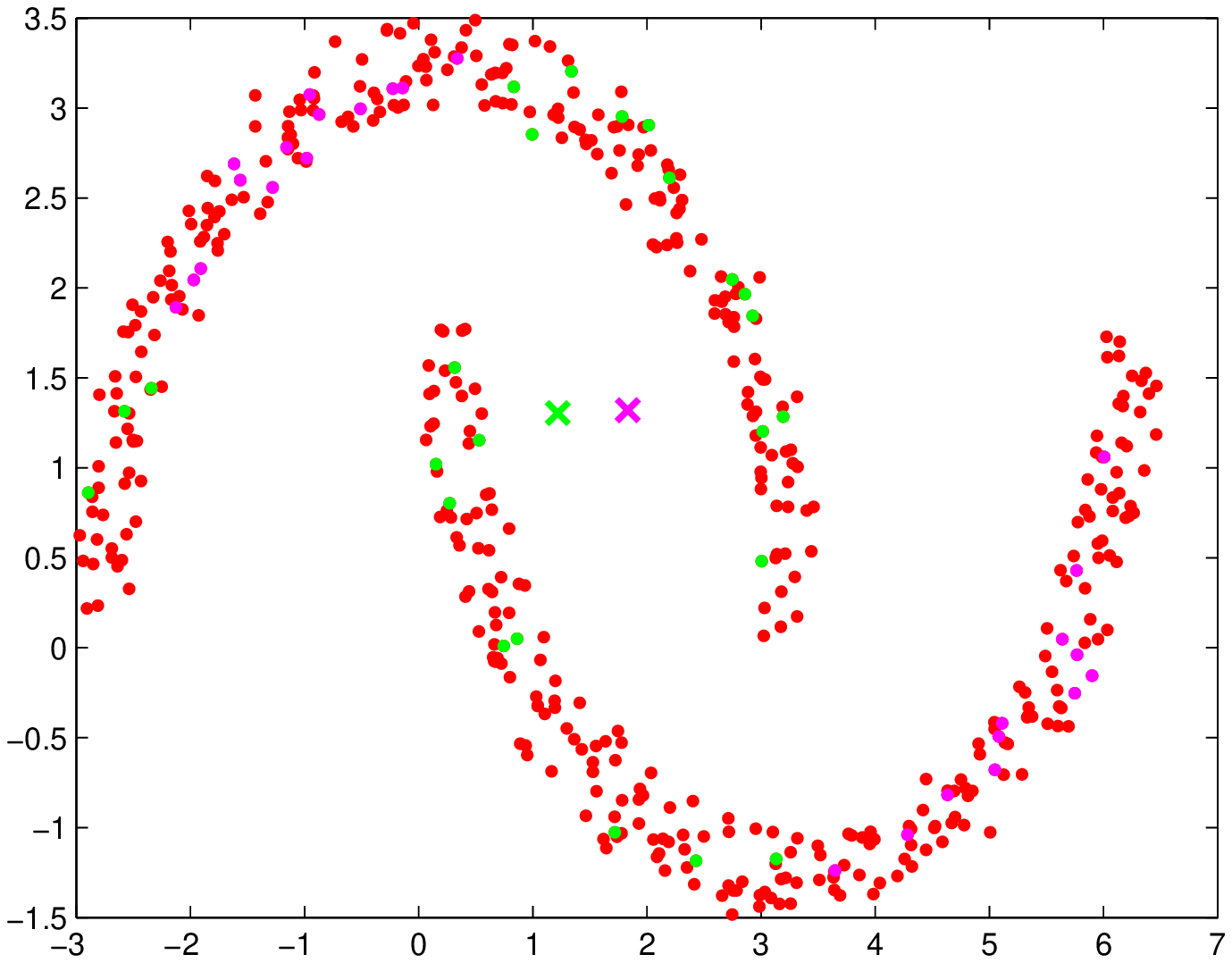}} &
\subfigure[]{\label{fig:synth8}\includegraphics[width=3cm,height=2.3cm]{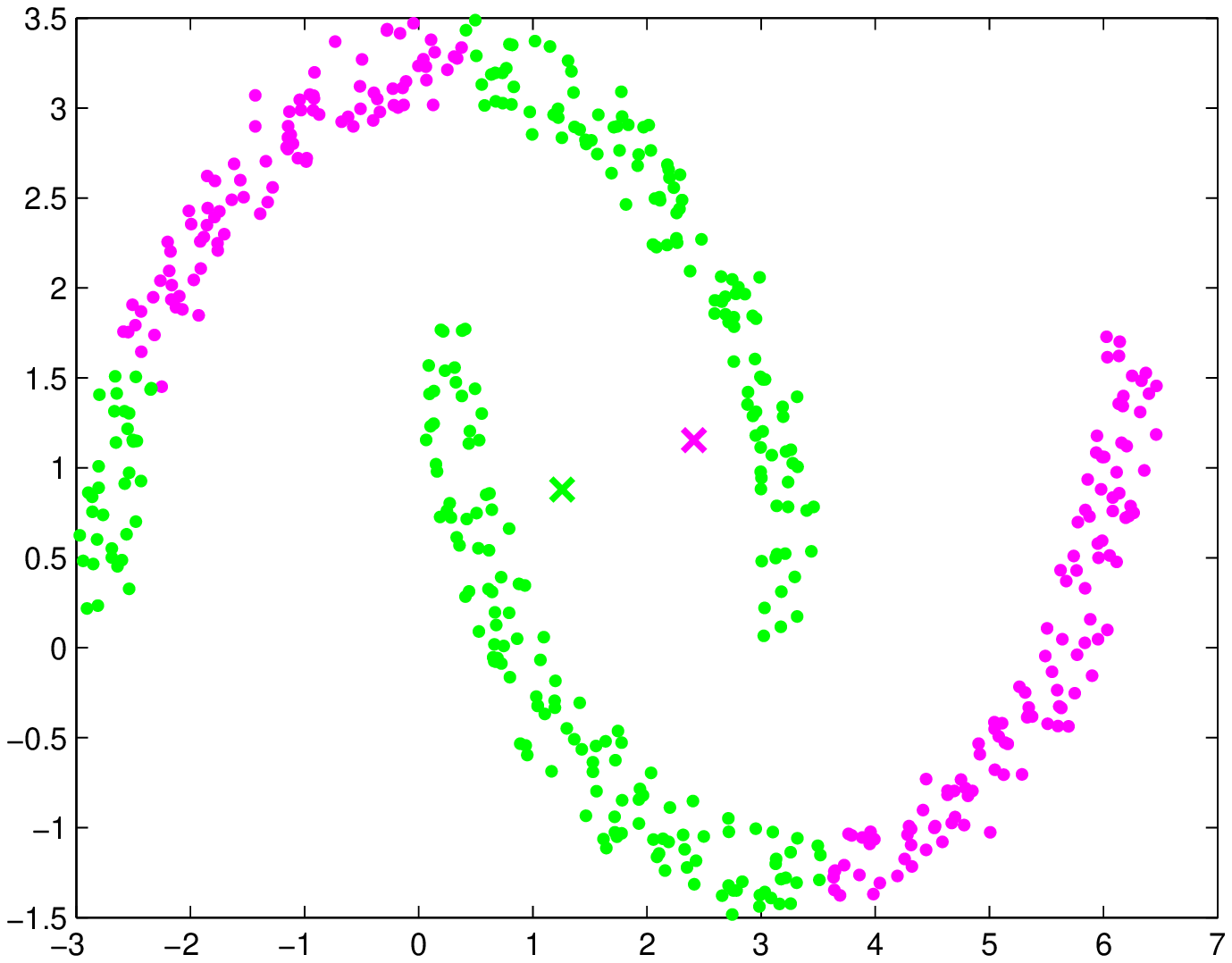}}
\end{tabular}
\caption{\footnotesize Illustration of aKKm and tKKm algorithms on a 2-D data set. Fig.~(a)
shows the 2-dimensional data containing 500 points from two clusters. Fig.~(b)
shows $m=50$ points sampled randomly from the data set (in blue).
Figs.~(c) - (e) and (f) - (h) show the process of finding two
clusters in the data set and their centers (represented by x) by the aKKm and tKKm algorithms, respectively. aKKm assigns labels to all the data points during each
iteration while constraining the centers to the subspace spanned by the sampled
points. tKKm clusters only the sampled points and then assigns
labels to the remaining points based on the cluster centers identified in the first step.}
\label{fig:synth_data}
\end{figure*}

\begin{algorithm}
\footnotesize
\center \caption{\footnotesize Two-step Kernel \textit{k}-means (tKKm)}
\begin{algorithmic}[1] \label{alg:2step_kkmeans}
    \REQUIRE {$ $}
    \begin{compactitem}
        \item $X = (x_1, \ldots, x_n)$: the set of $n$ data points to be
clustered
        \item $\kappa(\cdot, \cdot):\Re^d\times \Re^d \mapsto \Re$: kernel
function
        \item $m$: the number of randomly sampled data points ($m\ll n$)
        \item $C$: the number of clusters
    \end{compactitem}
    \ENSURE Cluster membership matrix $U \in \{0, 1\}^{C\times n}$ 
    \STATE Randomly select $m$ data points from $X$, denoted by $\Xh = (\xh_1,
\ldots, \xh_m)$.
    \STATE Compute the cluster centers, denoted by $c_k(\cdot), k \in [C]$, by
applying kernel \textit{k}-means to $\Xh$.
    \FOR{$i = 1, \ldots, n$}
        \STATE Update the $i^{th}$ column of $U $ by $U _{k_*i} = 1$ where
\[k_*=\mathop{\arg\min}\limits_{k \in [C]} |c_k(\cdot) - \kappa(x_i,
\cdot)|_{\Hk}.\]
    \ENDFOR
\end{algorithmic}
\end{algorithm}
\begin{algorithm}
\footnotesize
\center \caption{\footnotesize Approximate Kernel \textit{k}-means (aKKm)}
\begin{algorithmic}[1] \label{alg:approx_kkmeans}
    \REQUIRE {$ $}
    \begin{compactitem}
        \item $X = (x_1, \ldots, x_n)$: the set of $n$ data points to be
clustered
        \item $\kappa(\cdot, \cdot):\Re^d\times \Re^d \mapsto \Re$: kernel
function
        \item $m$: the number of randomly sampled data points ($m\ll n$)
        \item $C$: the number of clusters
	\item $MAXITER$: maximum number of iterations
    \end{compactitem}
    \ENSURE Cluster membership matrix $U \in \{0, 1\}^{C\times n}$ 
    \STATE Randomly sample $m$ data points from $X$, denoted by $\Xh = (\xh_1,
\ldots, \xh_m)$.
    \STATE Compute $K_B = [\kappa(x_i, \xh_j)]_{n\times m}$ and $\Kh =
[\kappa(\xh_i, \xh_j)]_{m\times m}$.
    \STATE Compute $T = K_B\Kh^{-1}$.
    \STATE Randomly initialize the membership matrix $U $. 
    \STATE Set $t=0$.
    \REPEAT
	\STATE Set $t = t + 1$.
        \STATE Compute the $\ell_1$ normalized membership matrix $\Uh$ by $\Uh =
[\diag(U \mathbf{1})]^{-1} U $.
        \STATE Calculate $\alpha = \Uh T$.
        \FOR{$i = 1, \ldots, n$}
            \STATE Update the $i^{th}$ column of $U $ by $U _{k_*i} = 1$ where
\[k_*=\mathop{\arg\min}\limits_{k \in [C]} \alpha^{\top}_k \Kh
\alpha_k - 2\varphi_i^{\top}\alpha_k.\]
where $\alpha_j$ and $\varphi_j$ are the $j^{th}$ rows of matrices $\alpha$ and $K_B$, respectively.
        \ENDFOR
    \UNTIL the membership matrix $U$ does not change or $t > MAXITER$
\end{algorithmic}
\end{algorithm}
\begin{figure}
\centering
\includegraphics[width = 4.6cm,height=3cm]{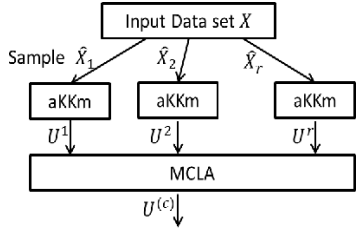}
\caption{\footnotesize Ensemble Approximate Kernel
\textit{k}-means}
\label{fig:ensemble_approx_kkmeans}
\end{figure}

\section{Ensemble Approximate Kernel \textit{k}-means}
\label{sec:ensemble_approxkkmeans}
We improve the quality of the aKKm solution by
using ensemble clustering. 

The objective of ensemble clustering~\cite{vega2010survey} is to 
combine multiple partitions of the given data set. A popular ensemble 
clustering algorithm is the Meta-Clustering algorithm (MCLA)~\cite{strehl2003cluster}, which maximizes
the average normalized mutual information. It is based on 
hypergraph partitioning. Given $r$ cluster membership matrices, $\{U^1,\ldots,
U^r\}$, 
where $U^{q}  = (\u_1^{q}, \ldots, \u_C^{q})^{\top}$, the objective 
of this algorithm is to find a consensus membership matrix $U^{(c)}$ that
maximizes the 
Average Normalized Mutual Information, defined as 
\begin{equation}\label{eqn:anmi}
 ANMI = \frac{1}{r}\mathop{\sum}\limits_{q=1}^{r}NMI(U^{(c)},U^{q}), 
\end{equation}
where $NMI(U^{a},U^{b})$, the Normalized Mutual Information
(NMI)~\cite{kvalseth1987entropy} between two partitions $a$ and $b$, represented
by the membership matrices $U^{a}$ and $U^{b}$ respectively, is defined by
\begin{equation}
 NMI(U^{a},U^{b}) =
\frac{\mathop{\sum}\limits_{i=1}^{C}\mathop{\sum}\limits_{j=1}^{C}n^{a,b}_{i,j}
\log\left(\frac{n.n^{a,b}_{i,j}}{n_{i}^{a}n_{j}^{b}}\right)}{\sqrt{\left(\mathop
{\sum}\limits_{i=1}^{C}n_{i}^{a}\log\frac{n_{i}^{a}}{n}\right)\left(\mathop{\sum
}\limits_{j=1}^{C}n_{j}^{b}\log\frac{n_{j}^{b}}{n}\right)}}.
\label{eqn:nmi}
\end{equation}
In equation \eqref{eqn:nmi}, $n_{i}^{a}$ represents the number of data points
that have been assigned label $i$ in partition $a$, and $n^{a,b}_{i,j}$
represents the number of data points that have been assigned label $i$ in
partition $a$ and label $j$ in partition $b$. 
NMI values lie in the range $[0,1]$. An NMI value of 1 indicates perfect
matching between the two partitions whereas 0 indicates perfect mismatch.

Maximizing \eqref{eqn:anmi} is a combinatorial optimization problem and solving
it exhaustively is 
computationally infeasible. MCLA obtains an approximate
consensus solution by representing the set of partitions as a hypergraph. 
Each vector $\u_k^{q}, k \in [C], q \in [r]$ represents a vertex in a regular
undirected graph, called the \emph{meta-graph}. 
Vertex $\u_i$ is connected to vertex $\u_j$ by an edge whose weight is
proportional to the Jaccard similarity between the 
two vectors $\u_i$ and $\u_j$:
\begin{equation}
 s_{i,j} = \frac{\u_i^{\top}\u_j}{\|\u_i\|^2 + \|\u_j\|^2 - \u_i^{\top}\u_j}.
\label{eqn:jacc_sim}
\end{equation}
This meta-graph is partitioned using a graph partitioning algorithm such as
METIS~\cite{karypis1998software} 
to obtain $C$ balanced meta-clusters $\pi_1, \pi_2, \ldots \pi_C$. Each
meta-cluster $\pi_k = \left\lbrace \u_k^{(1)}, \u_k^{(2)}, \ldots
\u_k^{(s_k)}\right \rbrace$,
 containing $s_k$ vertices, is represented by the mean vector 
\begin{equation}
\mu_k = \frac{1}{s_k}\mathop{\sum}\limits_{i=1}^{s_k}\u_k^{(i)}.
\label{eqn:ensemble_mean}
\end{equation}
The value $\mu_{ki}$
represents 
the association between data point $x_i$ and the $k^{th}$ cluster. Each data
point $x_i$ is assigned to the meta-cluster 
with which it is associated the most, breaking ties randomly, i.e 
\begin{eqnarray}
 U^{(c)}_{k_{*}i}
= \left\{ \begin{array}{rl}
1 & \mbox{ if } k_{*} = \mathop{\arg\max}\limits_{k \in [C]} \mu_{ki}\\
0 & \mbox{ otherwise}
\end{array} \right.
\label{eqn:ensembleU}
\end{eqnarray} 

The aKKm algorithm is combined with MCLA to enhance its accuracy. We execute the aKKm algorithm $r$ times with
different samples from the data set 
and MCLA is used to integrate the partitions obtained from each execution into
a 
consensus partition.

More specifically, we independently draw $r$ samples $\{\Xh^1,\ldots,\Xh^r\}$,
where each 
$\Xh^i = \{\xh_1^i,\ldots,\xh_m^i\}$ contains $m$ data points. Each $\xh_j^i$ is uniformly sampled
without replacement. After the sampling 
of $\Xh^i$ is performed, the samples are replaced and the next sample $\Xh^j$ is
obtained. For each sample $\Xh^i$, 
we first compute the kernel matrices $K_B^i = [\kappa(x_a,\xh_b^i)]_{n\times m}$ and 
$\Kh^i = [\kappa(\xh_a^i, \xh_b^i)]_{m\times m}$, and then execute
aKKm to obtain the cluster membership matrix
$U^i$. We then combine the partitions $\{U^i\}_{i=1}^r$ 
using MCLA to
obtain the consensus cluster 
membership $U^{(c)}$.

This ensemble clustering algorithm is described in
Algorithm~\ref{alg:ensemble_approx_kkmeans} and illustrated in
Fig.~\ref{fig:ensemble_approx_kkmeans}. On the synthetic data set in Fig.~\ref{fig:synth1}, we obtain the partitions
similar to Fig.~\ref{fig:synth5} by using a sample 
of $20$ data points instead of $50$ data points. This illustrates that the
efficiency of the algorithm is improved by using ensemble clustering.

\begin{algorithm}
\footnotesize
\center \caption{\footnotesize Ensemble aKKm}
\begin{algorithmic}[1] \label{alg:ensemble_approx_kkmeans}
    \REQUIRE {$ $}
    \begin{compactitem}
	\item $X = (x_1, \ldots, x_n)$: the set of $n$ data points to be
clustered
        \item $\kappa(\cdot, \cdot):\Re^d\times \Re^d \mapsto \Re$: kernel
function
        \item $m$: the number of randomly sampled data points ($m\ll n$)
        \item $C$: number of clusters
	\item $r$: number of ensemble partitions
	\item $MAXITER$: maximum number of iterations
    \end{compactitem}
    \ENSURE Consensus cluster membership matrix $U^{(c)} \in \{0, 1\}^{C\times
n}$ 
    
    \FOR {$i = 1, \ldots, r$}
	\STATE Randomly select $m$ data points from $X$, denoted by $\Xh^{i}$.
	\STATE Run Algorithm~\ref{alg:approx_kkmeans} using  $\Xh^{i}$ as
the sampled points and obtain the 
	  cluster membership matrix $U^{i}$.
    \ENDFOR
    \newline\textbf{MCLA:}
    \STATE Concatenate the membership matrices $\{U^{i}\}_{i=1}^{r}$ to obtain
an $rC \times n$ matrix 
      $\mathcal{U} = (\u_1, \u_2, \ldots, \u_{rC})^{\top}$.

    \STATE Compute the Jaccard similarity $s_{i,j}$ between the vectors $\u_i$
and $\u_j$,
$i,j \in [rC]$ using \eqref{eqn:jacc_sim}.

    \STATE Construct a complete weighted meta-graph $G=(V,E)$, where vertex set
$V = \{\u_1, \u_2, \ldots, \u_{rC}\}$ and each edge $(\u_i, \u_j)$ is weighted
by $s_{i,j}$.
    \STATE Partition $G$ into $C$ meta-clusters $\{\pi_k\}_{k=1}^{C}$ and
compute the mean vectors 
$\{\mu_k\}_{k=1}^{C}$ using \eqref{eqn:ensemble_mean}.
    \FOR {$i = 1, \ldots, n$}
      \STATE Update the $i^{th}$ column of $U^{(c)}$ in accordance with
\eqref{eqn:ensembleU}.
    \ENDFOR
\end{algorithmic}
\end{algorithm}
\section{Analysis of Approximate Kernel \textit{k}-means}
In this section, we first show that the computational complexity of the aKKm 
algorithm is less than that of the kernel \textit{k}-means algorithm and then
we derive bounds on
the difference in the clustering error achieved by our algorithm and the kernel
\textit{k}-means algorithm. 
\subsection{Computational complexity}
\emph{Approximate kernel \textit{k}-means:} The aKKm algorithm consists of two parts: kernel computation and clustering. As only an $n \times m$ portion of the 
kernel needs to be computed and stored, the cost of kernel computation is $O(ndm)$, a dramatic reduction over the $O(n^2d)$ complexity of classical kernel \textit{k}-means. 
The memory requirement also reduces to $O(mn)$.
 The most expensive clustering operation is the matrix inversion
$\Kh^{-1}$ and calculation of $T =  K_B\Kh^{-1}$, which has a computational cost
of $O(m^3 + m^2n)$. The cost of computing $\alpha$ and updating
the membership matrix $U$ is $O(mnCl)$, where $l$ is the number of iterations
needed for convergence. Hence, the overall cost of clustering is $O(m^3 + m^2 n
+ mnCl)$. We can further reduce this cost by avoiding the
matrix inversion $\Kh^{-1}$ and formulating the calculation of $\alpha = \Uh T =
\Uh K_B\Kh^{-1}$ as the following optimization problem:
\begin{eqnarray}
        \min\limits_{\alpha \in \Re^{C\times m}} \frac{1}{2}\tr(\alpha \Kh
\alpha) - \tr(\Uh K_B\alpha^{\top}) \label{eqn:update-alpha}
\end{eqnarray}
If $\Kh$ is well conditioned (i.e., the minimum eigenvalue of $\Kh$ is
significantly larger than zero), we can solve the optimization problem in
\eqref{eqn:update-alpha} using the simple gradient descent method with a convergence
rate of $O\left(\log(1/\varepsilon)\right)$, where $1 - \varepsilon$ is the
desired
accuracy. As the cost of each step in the gradient descent method
is $O(m^2 C)$, the overall computational cost is only $O(m^2 C l
\log(1/\varepsilon)) \ll O(m^3)$ when $Cl \ll m$. Using this approximation, we
can reduce
the overall computational cost to $O(m^2Cl + mnCl + m^2n) \sim O(n)$ when $m \ll n$. \newline\newline
\emph{Ensemble aKKm:} In the ensemble aKKm algorithm, an additional cost of $O(nC^2r^2)$ is incurred for combining the partitions using MCLA~\cite{strehl2003cluster}. 
However, we empirically observed that the sample size $m$ required to achieve a satisfactory 
clustering accuracy is reduced considerably when compared to aKKm. This leads to a further reduction in the running time.
\newline\newline

\subsection{Clustering error}
Let binary random variables $\xi = (\xi_1,\xi_2,...,\xi_n)^\top \in
\{0,1\}^n$ represent the random sampling process, where $\xi_i = 1$ if $x_i
\in \Xh$ and $0$ otherwise.
The following proposition allows us to write the clustering error in terms of 
the random variable $\xi$. 
\begin{prop}
Given the cluster membership matrix $U  = (\u_1, \ldots, \u_C)^{\top}$, the
clustering error can be expressed as
\begin{equation}
\label{sample_kkmeans_loss}
\L(U ,\xi) = \tr(K) + \sum_{k=1}^C \L_k(U , \xi),
\end{equation}
where $\L_k(U , \xi)$ is
\[
\L_k(U , \xi) = \min\limits_{\alpha_k \in \Re^{n}} - 2 \u_k^{\top}K(\alpha_k
\circ \xi) + n_k(\alpha_k \circ \xi)^{\top} K (\alpha_k \circ \xi).
\]
\end{prop}
Note that $\xi = \mathbf{1}$, where $\mathbf{1}$ is a vector of all ones,
implies that all the data points are chosen for constructing the subspace
$\H_b$, which is equivalent to kernel \textit{k}-means using the full kernel
matrix. As a result, $\L(U , \mathbf{1})$ is the clustering error of the
standard kernel \textit{k}-means algorithm.

The following theorem bounds the expectation of the clustering error.
\begin{thm} \label{thm:err_bound}
Given the membership matrix $U $, we have the expectation of $\L(U , \xi)$
bounded as follows
\begin{eqnarray*}
    \E_{\xi}[\L(U , \xi)] &\leq& \L(U , \mathbf{1}) \\
                    &+& \tr\left(\Ut\left[K^{-1} +
\frac{m}{n}[\diag(K)]^{-1}\right]^{-1}\Ut^{\top}\right),
\end{eqnarray*}
where $\L(U , \mathbf{1}) = \tr(K) - \tr(\Ut K\Ut^{\top})$.
\end{thm}
\begin{proof}
We first bound $\E_{\xi}[\L_k(U , \xi)]$ as
\begin{eqnarray*}
& & \frac{1}{n_k}\E_{\xi}[\L_k(U , \xi)] \\
& = & \E_\xi\left[\min_{\alpha} - 2\uh_k^{\top}K (\alpha  \circ \xi) + (\alpha
\circ \xi)^{\top}K(\alpha \circ \xi)\right] \\
& \leq & \min_\alpha  \E_\xi\left[- 2\uh_k^{\top}K(\alpha \circ \xi) + (\alpha
\circ \xi)^{\top} K (\alpha \circ \xi)\right] \\
& = & \min_\alpha - 2\frac{m}{n}\uh_k^{\top} K \alpha +
\frac{m^2}{n^2}\alpha^{\top}K\alpha \\
&+& \frac{m}{n}\left(1 -
\frac{m}{n}\right)\alpha^{\top} \diag(K) \alpha \\
& \leq & \min_\alpha - 2\frac{m}{n}\uh_k^{\top} K \alpha +
\frac{m}{n}\alpha^{\top}\left( \frac{m}{n} K + \diag(K)\right) \alpha.
\end{eqnarray*}
By minimizing over $\alpha$, we obtain
\[
    \alpha_* = \left({\dfrac{m}{n}}K + \diag(K)\right)^{-1} K \uh_k.
\]
Thus, $\E_{\xi}[\L_k(U , \xi)]$ is bounded as
\begin{eqnarray*}
& & \E_\xi[\L_k(U , \xi)] + n_k \uh_k^{\top} K \uh_k \\
& \leq & n_k\uh_k^{\top}\left(K - K\left[K +
{\dfrac{n}{m}}\diag(K)\right]^{-1}K\right)\uh_k \\
& = & \ut_k^{\top}\left(K^{-1} + {\dfrac{m}{n}}[\diag(K)]^{-1}\right)^{-1}\ut_k.
\end{eqnarray*}
We complete the proof by adding up $\E_{\xi}[\L_k(U , \xi)]$ and using the fact
that
\begin{eqnarray*}
\L_k(U , \mathbf{1}) = \min\limits_{\alpha} - 2\u_k^{\top}K\alpha +
n_k\alpha^{\top}K\alpha = - \ut_k^{\top} K \ut_k.
\end{eqnarray*}
\end{proof}
The following corollary interprets the result of the above theorem in terms of
the eigenvalues of $K$. Let $\lambda_1 \geq \lambda_2 \geq
\ldots \geq \lambda_n \geq 0$ and $Z=(z_1, z_2,\ldots,z_n)$ be the eigenvalues and the corresponding eigenvectors of $K$.
\setcounter{thm}{0}
\begin{cor}
Given the
membership matrix $U $, we have
\begin{eqnarray*}
\frac{\E_{\xi}[\L(U , \xi)]}{\L(U , \mathbf{1})} &  \leq & 1 +
\frac{\sum_{i=1}^C \lambda_i /[1 + \lambda_im/n]}{\tr(K) - \sum_{i=1}^C
\lambda_i} \\
& \leq & 1 + \frac{C/m}{\sum_{i=C+1}^n \lambda_{i}/n}.
\end{eqnarray*}
\end{cor}
\begin{proof}
As $\kappa(x, x) \leq 1$ for any $x$, we have $\diag(K) \preceq I$, where $I$ is
an identity matrix. As $\Ut$ is an $\ell_2$ normalized matrix, we have
\begin{eqnarray*}
& & \tr\left(\Ut\left[K^{-1} +
\frac{m}{n}[\diag(K)]^{-1}\right]^{-1}\Ut^{\top}\right) \\
& \leq & \tr\left(\Ut\left[K^{-1} + \frac{m}{n}I\right]^{-1}\Ut^{\top}\right) \\
& \leq & \sum_{i=1}^C \frac{\lambda_i}{1 + m\lambda_i/n} \leq \frac{Cn}{m}
\end{eqnarray*}
and
\[
\L(U , \mathbf{1}) = \tr(K - U  K U ^{\top}) \geq \tr(K) - \sum_{i=1}^C
\lambda_i.
\]
We complete the proof by combining the above inequalities.
\end{proof}

As an illustration of the result of Corollary 1, consider a special kernel
matrix $K$
that has its first $a$ eigenvalues equal to $n/a$ and the remaining eigenvalues
equal to zero; i.e. $\lambda_1= \ldots = \lambda_a = n/a$ and $\lambda_{a+1} =
\ldots = \lambda_{n} = 0$. Assuming $a > 2C$, i.e. the number of
non-zero eigenvalues of $K$ is larger than twice the number of clusters, we have
\begin{equation}
    \frac{\E_{\xi}[L(U , \xi)] - \L(U , \mathbf{1})}{\L(U , \mathbf{1})} \leq 1
+ \frac{C a}{m(a - C)} \leq 1 + \frac{2C}{m}.
\label{eqn:err_bound}
\end{equation}
This indicates that when the number of non-zero eigenvalues of $K$ is
significantly
larger than the number of clusters, the difference in the clustering errors
between the standard kernel \textit{k}-means and our approximation scheme decreases at the rate of $O(1/m)$.
\subsection{Parameter sensitivity}
One of the important factors that determines the performance of the approximate
kernel \textit{k}-means 
algorithm is the sample size $m$. Sampling introduces a trade-off between 
clustering quality and efficiency. As $m$ increases, the clustering quality
improves but the speedup achieved by the algorithm suffers. The following theorem gives an estimate of $m$. 

\begin{thm}
\label{thm:m_est}
Let $\Sigma_1 = \diag(\lambda_1,\ldots,\lambda_C)$, $\Sigma_2 = \diag(\lambda_{C+1},\ldots,\lambda_n)$, $Z_1=(z_1, \ldots,z_C)$ and $Z=(z_{C+1},\ldots,z_n)$.
Let \begin{equation}\label{eqn:coherence}\tau=n\max \limits_{1 \leq i \leq n} |Z_{(i)}|^2\end{equation} denote the coherence of the kernel matrix $K$ (adapted from ~\cite{gittens2011spectral}).
For any $\epsilon \in (0,1)$, the spectral norm of the approximation error is bounded above, with probability $1 - \delta$, as 
\begin{equation*}
 \|K - K_B \Kh^{-1} K_B^{\top}\|_2 \leq \lambda_{C+1}\left(1 + 8 \tau \ln\frac{2}{\delta}\sqrt{\frac{Cn}{m}}\right),
\end{equation*}
provided $m \geq \tau C \max(C_1 \ln p , C_2 \ln (3/\delta))$, for some positive constants $C_1$ and $C_2$.
\end{thm}
A small sample size $m$ suffices if the coherence measure $\tau$ is low and there is a large gap in the eigenspectrum. The approximation error reduces at a rate 
of $O(1/\sqrt{m})$\footnote{This improves over the approximation error bound of the naive Nystrom method presented in~\cite{gittens2011spectral}.}. The reader is referred to the appendix for the proof of this theorem.

In our experiments, we examined the performance of our algorithm for
different sample sizes ($m$) ranging from $0.001\%$ to $15\%$ of the data set size $n$,
and observed that setting $m$ equal to $0.01\%$ to $0.05\%$ of $n$ leads to 
a satisfactory performance.
\section{Experimental Results}
In this section, we show that aKKm is an efficient and
scalable variant of the kernel \textit{k}-means algorithm. It has lower run-time
and memory
requirements but is on par with kernel \textit{k}-means in terms of the
clustering error and clustering quality. We tested our algorithm on four data sets with varying
sizes:  Imagenet, MNIST, Forest Cover Type, and Network Intrusion data sets
(Table~\ref{tbl:datasets}). Using
small
and medium-sized data sets (Imagenet and MNIST) for which the full kernel
calculation is feasible on a single processor, we demonstrate that our algorithm's clustering
performance is similar to that of the kernel \textit{k}-means algorithm. We then
demonstrate
scalability using the large Forest Cover Type and Network Intrusion
data sets. 
Finally, we show that the performance can be improved by using the ensemble aKKm algorithm. 

All algorithms were implemented in MATLAB\footnote{We used the 
\textit{k}-means implementation in the MATLAB Statistics Toolbox and 
the Nystrom approximation based spectral clustering implementation~\cite{chen2011parallel} available at 
\url{http://alumni.cs.ucsb.edu/~wychen/sc.html}. The remaining 
algorithms were implemented in-house.} and run on an 2.8 GHz processor. The
memory used was explicitly limited to
$40$ GB.

\begin{table}
\centering
 \scriptsize
    \begin{tabular}{|c|c|c|c|}
   \hline
   \textbf{Data set} & \textbf{Size} & \textbf{Dimensionality} & \textbf{Kernel function} \\
   \hline
   Imagenet & 20,000 & 17,000 & Pyramid \\
   \hline
   MNIST & 70,000 & 784 & Neural \\
  \hline
  Forest Cover Type & 581,012 & 54 & RBF \\
  \hline
  Network Intrusion & 4,898,431 & 50 & Polynomial \\
\hline
\end{tabular}
\caption{\footnotesize Data set summary}
\label{tbl:datasets}
\end{table}

\subsection{Performance comparison with kernel \textit{k}-means}
We use the Imagenet and MNIST data sets to demonstrate that the approximate
kernel \textit{k}-means algorithm's 
clustering performance is similar to that of the kernel \textit{k}-means algorithm.
\subsubsection{Datasets}
\begin{compactitem}
\item \textbf{Imagenet:} The Imagenet data set~\cite{deng2009imagenet} 
consists of over 1.2 million images that are organized according to the WordNet
hierarchy. Each node in this hierarchy represents a concept (known as the
``synset''). We chose $20,000$ images
from $12$ synsets.
We extracted
keypoints from each image
using the VLFeat library~\cite{vedaldi08vlfeat} and represented each keypoint as
a $128$ dimensional SIFT descriptor; an average of $3,055$ keypoints were
extracted from each image. 
\item \textbf{MNIST:} The MNIST data set~\cite{lecun1998gradient} is a subset of
the database of handwritten digits available from NIST. It contains $60,000$
training images and $10,000$ test images from $10$ classes.
Each image is represented using a $784$-dimensional feature vector. For the
purpose of
clustering, we combined the training and test images to
form a data set with $70,000$ images. 
\end{compactitem}
\subsubsection{Experimental setup} 
We first compare the approximate kernel \textit{k}-means (aKKm) algorithm with the kernel \textit{k}-means
algorithm 
to show that they achieve similar performance. We then compare
it to the 
two-step kernel \textit{k}-means (tKKm) algorithm and demonstrate that aKKm is superior. We
also gauge it's performance against that of (i) the Nystrom spectral clustering
algorithm (nysSC)~\cite{fowlkes2004spectral}, which 
clusters the top $C$ eigenvectors of a low rank approximate kernel
obtained through the Nystrom approximation technique, (ii) the leaders based kernel 
\textit{k}-means algorithm (lKKm)~\cite{sarma2012speeding} which finds a few representative patterns (called leaders) 
based on a user-defined distance threshold and then runs kernel \textit{k}-means on the leaders, and (iii) the 
\textit{k}-means algorithm to show that it achieves a better 
clustering accuracy.

For the Imagenet data set, we employ the spatial pyramid
kernel~\cite{lazebnik2006beyond} to calculate the pairwise similarity with the
number of pyramid levels set to be $4$. This has been shown to be effective for
object recognition and image retrieval. It took $24,236$ seconds to compute
the 
multi-resolution histograms and the pyramid representation of the data.   

On the MNIST data set, we use the neural
kernel defined as $\kappa(x,y)=tanh(ax^{\top}y + b)$, with the parameters $a$
and 
$b$ set to $0.0045$ and $0.11$ respectively, as suggested
in~\cite{zhang2002large}.

We evaluate the efficiency of the aKKm algorithm for
different sample sizes 
ranging from $100$ to $5,000$. We directly compute $\Kh^{-1}$ instead of using the
approximation method in 
\eqref{eqn:update-alpha} to demonstrate that our algorithm is efficient in spite of 
a naive implementation. 
The number of clusters $C$ is set equal to the number of true classes 
in the data set. 

We measure the time taken for computing the
kernel matrix (when applicable) and clustering the data points. We measure the
clustering
performance using \textit{error reduction}, defined as the ratio of the
difference between the initial clustering error (on random initialization) and
the final clustering error (after running the clustering algorithm) to the
initial clustering error. The larger the error reduction, the greater is the cluster
compactness. To evaluate the difference between the clustering results of the aKKm and tKKm algorithms, and the kernel \textit{k}-means, we calculate the
\textit{Adjusted Rand Index} (ARI)~\cite{hubert1985comparing}, a measure of
similarity between two data partitions. The adjusted Rand index value lies in
$[0,1]$. A value close to $1$ indicates better matching between the two
partitions than a value close to $0$. Finally, we measure the clustering accuracy in terms of the NMI
with respect to the true class labels. All the results are averaged over
10 runs of the algorithms.

\subsubsection{Experimental results}
\textbf{Imagenet:} 
Tables~\ref{tbl:time_imagenet} and~\ref{tbl:ari}, and
Figs.~\ref{fig:error_red_imagenet} and~\ref{fig:nmi_imagenet} compare the
performance
of the aKKm algorithm on the Imagenet data set with the kernel
\textit{k}-means, \textit{k}-means, lKKm, tKKm, and nysSC clustering algorithms.

Table~\ref{tbl:time_imagenet} lists the running time of all the algorithms.
The kernel computation time is common to both the aKKm
and tKKm algorithms. We observe
that a significant speedup (over $90\%$) is achieved by both the algorithms 
in kernel computation when compared to kernel \textit{k}-means. tKKm is the most efficient in terms of the
clustering speedup. aKKm takes longer as it needs to compute the inverse matrix $\Kh^{-1}$. However, when the
algorithms are compared with the requirement that they yield the same clustering
performance, we will see later that aKKm is more efficient. For $m \leq 1,000$, 
aKKm is even faster than the \textit{k}-means algorithm executed on the
pyramid features. This is due to the high dimensionality of the pyramid features, which plays an important role in the complexity of \textit{k}-means.

In Table~\ref{tbl:ari} (columns 2-3), we observe that the ARI values of the aKKm algorithm are much higher than those of the tKKm algorithm. This 
shows that aKKm obtains partitions that are similar 
to the kernel \textit{k}-means partitions. 
When $m=5,000$, the partitions obtained by aKKm are, on an average,
$84\%$ similar 
to the partitions generated by kernel \textit{k}-means.  

Figs.~\ref{fig:error_red_imagenet} and~\ref{fig:nmi_imagenet} show the error
reduction and NMI plots, respectively. 
We again observe that the aKKm algorithm achieves performance similar to that of the
kernel \textit{k}-means. The tKKm algorithm achieves much lower error reduction.
The tKKm, lKKm and nysSC algorithms yield lower NMI values. 
With just $100$ sampled points, aKKm significantly
outperforms tKKm provided with $1,000$ sampled points, and achieves
the same performance as the nysSC algorithm. This
observation indicates that it is
insufficient to estimate the cluster centers using only the randomly sampled
data points as in the tKKm method, further justifying the design of the aKKm algorithm. As expected, 
all the kernel-based algorithms perform better than the \textit{k}-means
algorithm.


\textbf{MNIST:} Table~\ref{tbl:time_mnist} shows the results for the MNIST
data set using the neural kernel. Unlike the Imagenet data set, more time is
spent in 
clustering than in kernel calculation due to the simplicity of the kernel
function. 
As observed in the Imagenet data set, a significant amount of
time was saved by the aKKm algorithm as well as by the tKKm
algorithm, when compared to the kernel
\textit{k}-means algorithm. When the sample size is small ($m < 5,000$), 
the aKKm algorithm is also faster than the \textit{k}-means algorithm. 

Though the tKKm algorithm is more efficient than the aKKm algorithm, the
ARI values in 
Table~\ref{tbl:ari} (columns 4-5) indicate that the tKKm algorithm produces inferior partitions. The partitions generated by the aKKm 
algorithm are more 
similar to those generated by the kernel \textit{k}-means even for small
sample sizes. 
The aKKm algorithm achieves similar performance
as kernel \textit{k}-means when $m=500$, whereas the tKKm algorithm cannot 
achieve this until $m \geq 5,000$. 

As seen in Fig.~\ref{fig:error_red_mnist}, approximately equal amounts of error
reduction are achieved by
both the kernel \textit{k}-means and the aKKm algorithm for $m \geq
500$. In the NMI plot shown in Fig.~\ref{fig:nmi_mnist}, we first observe that all the
kernel-based algorithms, 
except tKKm and lKKm, perform better than the \textit{k}-means
algorithm. The aKKm algorithm's performance is better than that of 
the lKKm and nysSC algorithms, and comparable to that of kernel
\textit{k}-means, when $m \geq 500$.

\begin{table}
\centering
\subfigure[\scriptsize Imagenet]{
\label{tbl:time_imagenet}
  \scriptsize
    \begin{tabular}{|c|c|c|c|c|}
   \hline
\multicolumn{5}{|c|}{\textbf{\textit{k}-means clustering time:} 3,446.74 ($\pm793.57$)} \\
   \hline
\multicolumn{5}{|c|}{\textbf{lKKm clustering time:} 452.27 ($\pm85.10$)} \\
   \hline
\multicolumn{2}{|c|}{\textbf{Full Kernel}}&\multicolumn{3}{|c|}{\textbf{Kernel \textit{k}-means}}\\
\multicolumn{2}{|c|}{\textbf{calculation time:}}&\multicolumn{3}{|c|}{\textbf{clustering time:}}\\
\multicolumn{2}{|c|}{ 33,132.06}&\multicolumn{3}{|c|}{81.23 ($\pm43.66$)}\\
\hline
\textbf{Sample}&\textbf{Kernel} &\multicolumn{2}{|c|}{\textbf{Clustering time }}&\textbf{nysSC}\\
\cline{3-4}
\textbf{size}&\textbf{calculation}&\textbf{aKKm}&\textbf{tKKm}&\\
\hline
100 & 203.65&  4.33&0.02&18.27\\
&($\pm20.87$)&($\pm1.32$)&($\pm0.005$)&($\pm0.67$)\\
\hline
200 & 527.11& 7.67 & 0.69&14.59\\
&($\pm20.08$)&($\pm2.23$)&($\pm0.36$)&($\pm1.45$)\\
\hline
500 & 810.03 & 10.81&0.01&63.99\\
&($\pm24.36$)&($\pm2.31$)&($\pm0.001$)&($\pm0.79$)\\
\hline
1,000 & \textbf{1,564.26}& \textbf{34.56}&1.05&185.6\\
&\textbf{($\mathbf{\pm63.66}$)}&\textbf{($\mathbf{\pm4.20}$)}&($\pm0.28$)&($\pm43.13$)\\
\hline
2,000 & 3,867.10& 53.20&2.69 &565.84\\
&($\pm101.28$)&($\pm10.66$)&($\pm1.29$)&($\pm95.50$)\\
\hline
5,000 &14,165.98 & 450.42& 3.44&4,989.21\\
&($\pm899.49$)&($\pm97.89$)&($\pm0.79$)&($\pm403.06$)\\
\hline
 \end{tabular}}
\qquad
\subfigure[\scriptsize MNIST]{
\label{tbl:time_mnist}
\scriptsize
    \begin{tabular}{|c|c|c|c|c|}
\hline
\multicolumn{5}{|c|}{\textbf{\textit{k}-means clustering time:} 448.69 ($\pm177.24$)} \\
   \hline
\multicolumn{5}{|c|}{\textbf{lKKm clustering time:} 445.11 ($\pm26.25$)} \\
   \hline
\multicolumn{2}{|c|}{\textbf{Full Kernel}}&\multicolumn{3}{|c|}{\textbf{Kernel \textit{k}-means}}\\
\multicolumn{2}{|c|}{\textbf{calculation time:}}&\multicolumn{3}{|c|}{\textbf{clustering time:}}\\
\multicolumn{2}{|c|}{ 514.54}&\multicolumn{3}{|c|}{3,953.62 ($\pm2,157.69$)}\\
\hline
\textbf{Sample}&\textbf{Kernel} &\multicolumn{2}{|c|}{\textbf{Clustering time }}&\textbf{nysSC}\\
\cline{3-4}
\textbf{size}&\textbf{calculation}&\textbf{aKKm}&\textbf{tKKm}&\\
\hline
100&2.02&22.7&0.06&7.20\\
&($\pm0.09$)&($\pm13.30$)&($\pm0.002$)&($\pm1.00$)\\
\hline
200&5.34&53.48&0.18&49.56\\
&($\pm0.33$)&($\pm30.17$)&($\pm0.03$)&($\pm9.19$)\\
\hline
500&\textbf{5.37}&\textbf{72.24}& 0.25&348.86\\
&\textbf{($\mathbf{\pm0.53}$)}&\textbf{($\mathbf{\pm29.44}$)}&($\pm0.02$)&($\pm107.43$)\\
\hline
1,000& 7.88&75.26& 0.59 &920.34\\
&($\pm0.18$)&($\pm22.42$)&($\pm0.08$)&($\pm219.62$)\\
\hline
2,000 &6.87 & 85.20&1.56 &4,180.21\\
&($\pm0.40$)&($\pm25.19$)&($\pm0.08$)&($\pm385.82$)\\
\hline
5,000 &  41.07  & 1,478.85    &   14.63   &26,743.83\\
&($\pm5.12$)&($\pm178.97$)&($\pm6.58$)&($\pm12.32$)\\
\hline
\end{tabular}}
\qquad
\subfigure[\scriptsize Forest Cover Type]{
\label{tbl:time_covertype}
\scriptsize
\begin{tabular}{|c|c|c|c|c|}
 \hline
\multicolumn{5}{|c|}{\textit{k}-means clustering time: 29.28 ($\pm9.12$)} \\
   \hline
\textbf{Sample}&\textbf{Kernel} &\multicolumn{2}{|c|}{\textbf{Clustering time }}&\textbf{nysSC}\\
\cline{3-4}
\textbf{size}&\textbf{calculation}&\textbf{aKKm}&\textbf{tKKm}&\\
\hline
100	&1.40&17.70	 &0.18&10.35\\
	&($\pm0.29$)&	($\pm6.06$)	&($\pm0.05$)&($\pm1.44$)\\
\hline
200	&1.64&	22.57 &0.36&16.83\\
	&($\pm0.09$)&	($\pm12.39$)	&($\pm0.11$)&($\pm2.38$)\\
\hline
500	&\textbf{3.82}&	\textbf{28.56} &0.52&50.11\\
	&\textbf{($\mathbf{\pm0.03}$)}&	\textbf{($\mathbf{\pm11.61}$)}	&($\pm0.13$)&($\pm10.83$)\\
\hline
1,000	&11.14&	55.01 &0.76&137.26\\
	&($\pm0.68$)&	($\pm18.57$)	&($\pm0.15$)&($\pm40.88$)\\
\hline
2,000	&22.80&	134.68 &1.22&550.75\\
	&($\pm1.27$)&	($\pm26.10$)	&($\pm0.29$)&($\pm326.22$)\\
\hline
5,000	&64.11&	333.31 &2.48&6,806.52\\
	&($\pm6.66$)&	($\pm39.35$)	&($\pm0.33$)&($\pm3,966.61$)\\
\hline
\end{tabular}}
\qquad
\subfigure[\scriptsize Network Intrusion]{
\label{tbl:time_kddcup}
   \scriptsize
\begin{tabular}{|c|c|c|c|c|}
   \hline
\multicolumn{5}{|c|}{\textbf{\textit{k}-means clustering time:} 1,725.76 ($\pm544.28$)} \\
   \hline
\textbf{Sample}&\textbf{Kernel} &\multicolumn{2}{|c|}{\textbf{Clustering time }}&\textbf{nysSC}\\
\cline{3-4}
\textbf{size}&\textbf{calculation}&\textbf{aKKm}&\textbf{tKKm}&\\
\hline
100&	6.48	&	151.66&5.47&4,300.39\\
	& ($\pm0.84$) &	($\pm5.15$)	& ($\pm2.76$)&($\pm248.32$)\\
\hline
200&	11.63	&	231.55&16.17&16,349.15\\
	& ($\pm0.79$) &	($\pm7.09$)	& ($\pm7.33$)&($\pm5,392.63$)\\
\hline
500&	27.03	&	283.19& 35.05&56,174.24\\
	& ($\pm2.19$) &	($\pm15.30$)	& ($\pm12.27$)&($\pm874.61$)\\
\hline
1,000&	\textbf{52.14}	&\textbf{433.51}	&33.87&123,928.09\\
	& \textbf{($\mathbf{\pm2.68}$)} &	\textbf{($\mathbf{\pm16.12}$)}	& ($\pm11.82$)&($\pm2,449.48$)\\
\hline
2,000&	77.13	&	2,231.37&65.94&156,661.34\\
	& ($\pm14.21$) &	($\pm637.87$)	& ($\pm7.94$)&($\pm105.16$)\\
\hline
\end{tabular}}
\caption{\footnotesize Comparison of running time in seconds on the four data sets}
\label{tbl:time}
\end{table}
\begin{table}
\centering
\scriptsize
\begin{tabular}{|c|c|c|c|c|}
\hline
Data set & \multicolumn{2}{|c|}{Imagenet} & \multicolumn{2}{|c|}{MNIST} \\ 
\hline
\textbf{Sample size}  & \textbf{aKKm}  & \textbf{tKKm} & \textbf{aKKm}  & \textbf{tKKm} \\
\hline
100 &0.71&0.28&0.47&0.37\\
&($\pm0.05$)&($\pm0.06$)&($\pm0.14$)&($\pm0.05$)\\
\hline
200 &0.71&0.34&0.61&0.36\\
&($\pm0.09$)&($\pm0.04$)&($\pm0.24$)&($\pm0.05$)\\
\hline
500 &0.75&0.51&\textbf{0.69}&0.48\\
&($\pm0.11$)&($\pm0.04$)&\textbf{($\mathbf{\pm0.11}$)}&($\pm0.07$)\\
\hline
1,000 & 0.77&0.58 &0.70&0.58\\
&($\pm0.18$)&($\pm0.06$)&($\pm0.12$)&($\pm0.09$)\\
\hline
2,000 &0.78&0.59&0.71&0.61\\
&($\pm0.13$)&($\pm0.08$)&($\pm0.08$)&($\pm0.04$)\\
\hline
5,000 &\textbf{0.84}& 0.68& 0.71   &  0.70 \\
&\textbf{($\mathbf{\pm0.12}$)}&($\pm0.11$)&($\pm0.12$)&($\pm0.12$)\\
\hline
\end{tabular}
\caption{\footnotesize Adjusted Rand Index (ARI) with respect to the kernel \textit{k}-means partitions}
\label{tbl:ari}
\end{table}
\subsection{Performance of Approximate kernel \textit{k}-means using different sampling strategies}
Table~\ref{tbl:sampling_strategies_imagenet} and
Figs.~\ref{fig:error_red_imagenet_sampling_comparison}
and~\ref{fig:nmi_imagenet_sampling_comparison} compare the diagonal sampling, 
column norm sampling, and \textit{k}-means sampling strategies with the uniform
random sampling technique. In
Table~\ref{tbl:sampling_strategies_imagenet}, we assume that the $n
\times n$ kernel matrix is pre-computed and only include the time taken for
sorting the diagonal entries (for diagonal sampling), or computing the column
norms (for column norm sampling), and the time taken for choosing the first $m$
indices, in the sampling time. For the \textit{k}-means sampling, we show the
time taken to 
execute \textit{k}-means and find the representative sample. As expected, the
sampling time for all the
non-uniform sampling techniques is greater than the time required for random
sampling. Because of the high complexity of column norm sampling and
\textit{k}-means sampling, their sampling
time is significantly greater than those of the other two methods. 
Figs.~\ref{fig:error_red_imagenet_sampling_comparison}
and~\ref{fig:nmi_imagenet_sampling_comparison} 
show that the column-norm sampling produces inferior error reduction and NMI results, 
compared to the random sampling.  The diagonal and k-means sampling achieve similar 
results as random sampling. These results show that uniform random sampling is not 
only the most efficient way to use the approximate kernel \textit{k}-means, but it 
also produces partitions that are as good as or better than the more ``intelligent'' 
sampling schemes.

In Table~\ref{tbl:sampling_strategies_mnist} and
Figs.~\ref{fig:error_red_mnist_sampling_comparison} 
and~\ref{fig:nmi_mnist_sampling_comparison}, we gauge the various non-uniform
sampling techniques 
against the uniform random sampling strategy. As expected, they take much longer
than
uniform sampling. 
Similar error reduction and NMI values are achieved by all the 
sampling strategies when provided with a sufficiently large sample ($m \geq
2,000$). 
Therefore, the additional time spent for non-uniform sampling does not lead to
any significant 
improvement in the performance.
\begin{table}
\centering
\subfigure[\scriptsize Imagenet]{\label{tbl:sampling_strategies_imagenet}
\scriptsize
\begin{tabular}{|c|c|c|c|c|}
\hline
\textbf{Sample} & \multicolumn{4}{|c|}{\textbf{Sampling time (milliseconds)}}\\
\cline{2-5}
\textbf{size}  & \textbf{Uniform} & \textbf{Column norm}& \textbf{Diagonal} & \textbf{\textit{k}-means} \\
&\textbf{Random}  & \textbf{($\mathbf{\times e03}$)} & &\textbf{($\mathbf{\times e06}$)}\\
\hline
100 & 1.40 & 3.38 &2.93 &5.95 \\
&($\pm0.23$)&($\pm0.08$)&($\pm0.58$)&($\pm0.99$)\\
\hline
200 &1.14 & 5.83& 4.13& 9.75\\
&($\pm0.22$)&($\pm0.43$)&($\pm0.30$)&($\pm1.45$)\\
\hline
500 & 1.34 & 3.44 &2.76& 6.33\\
&($\pm0.08$)&($\pm0.05$)&($\pm0.46$)&($\pm0.47$)\\
\hline
1,000 & \textbf{1.54} & 3.51 &2.95&11.40\\
&\textbf{($\mathbf{\pm0.17}$)}&($\pm0.03$)&($\pm0.24$)&($\pm1.18$)\\
\hline
2,000 & 1.59& 3.76& 4.73&23.30\\
&($\pm0.21$)&($\pm0.16$)&($\pm0.75$)&($\pm1.18$)\\
\hline
5,000 & 1.67& 6.43& 6.34&41.30\\
&($\pm0.17$)&($\pm0.64$)&($\pm0.78$)&($\pm4.49$)\\
\hline
\end{tabular}}
\qquad
\subfigure[\scriptsize MNIST]{ \label{tbl:sampling_strategies_mnist}
\scriptsize
\begin{tabular}{|c|c|c|c|c|}
\hline
\textbf{Sample} & \multicolumn{4}{|c|}{\textbf{Sampling time (milliseconds)}}\\
\cline{2-5}
\textbf{size}  & \textbf{Uniform} & \textbf{Column norm}& \textbf{Diagonal} & \textbf{\textit{k}-means}\\
& \textbf{Random}  &\textbf{($\mathbf{\times e03}$)} &  & \textbf{($\mathbf{\times e06}$)}\\
\hline
100  &9.41&94.22&20.72&$3.83$\\
&($\pm1.74$)&($\pm3.97$)&($\pm7.19$)&($\pm0.542$)\\
\hline
200  &9.34&88.92&18.70&$2.62$\\
&($\pm1.16$)&($\pm4.44$)&($\pm6.10$)&($\pm0.254$)\\
\hline
500  &11.10&86.27&18.53&$7.82$\\
&($\pm3.81$)&($\pm0.94$)&($\pm3.99$)&($\pm3.42$)\\
\hline
1,000  &\textbf{8.41}&86.15&19.59&$5.88$\\
&\textbf{($\mathbf{\pm1.38}$)}&($\pm0.70$)&($\pm9.30$)&($\pm1.78$)\\
\hline
2,000&9.53&86.66&22.03&$4.91$\\
&($\pm1.94$)&($\pm0.85$)&($\pm6.44$)&($\pm0.207$)\\
\hline
5,000 &8.87&85.87&26.18&$14.01$ \\
&($\pm0.99$)&($\pm0.85$)&($\pm22.49$)&($\pm1.01$)\\
\hline
\end{tabular}}
\caption{\footnotesize Comparison of sampling time for different sampling
strategies}
\label{tbl:sampling_strategies}
\end{table}

\begin{figure*}
\centering
\begin{tabular}{cccc}
\multicolumn{4}{c}{\includegraphics[width=10cm,height=0.5cm]{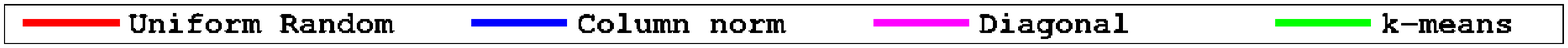}} \\
\subfigure[\scriptsize Error reduction]{\label{fig:error_red_imagenet_sampling_comparison}\includegraphics[width=4.2cm,height=3.2cm] {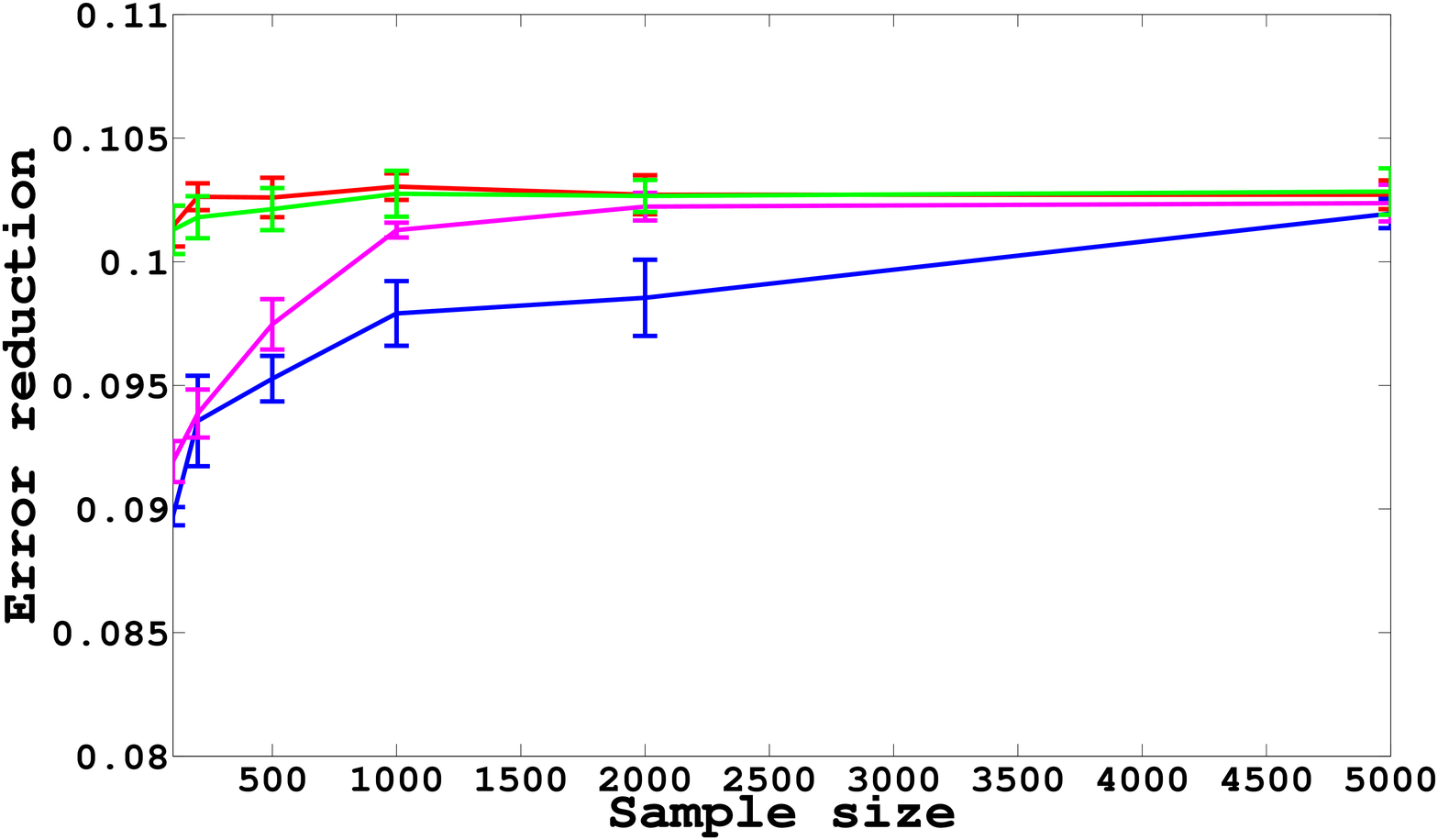}} &
\subfigure[\scriptsize NMI]{\label{fig:nmi_imagenet_sampling_comparison}\includegraphics[width=4.2cm,height=3.2cm]{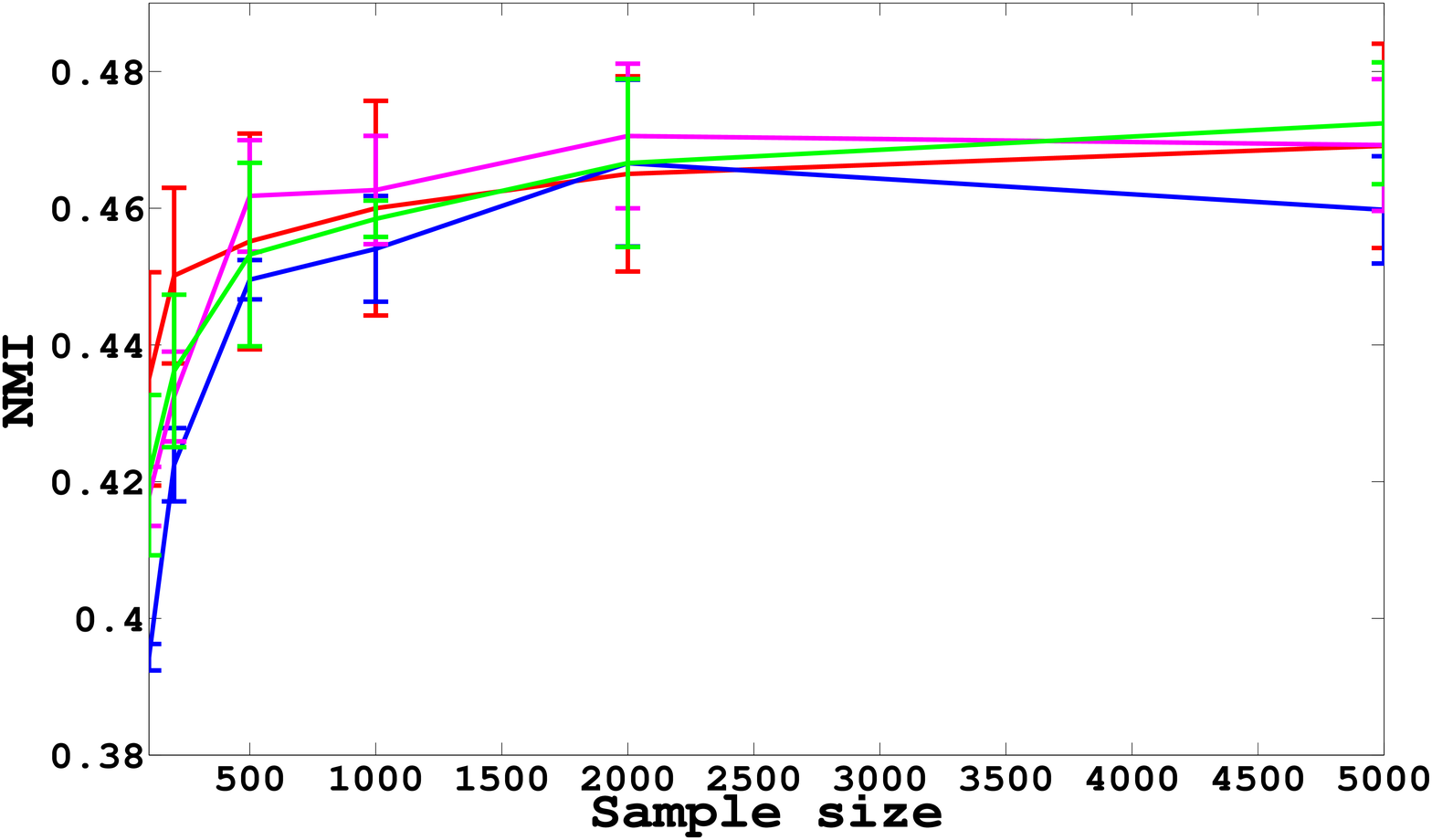}} &
\subfigure[\scriptsize Error reduction]{\label{fig:error_red_mnist_sampling_comparison}\includegraphics[width=4.2cm,height=3.2cm]{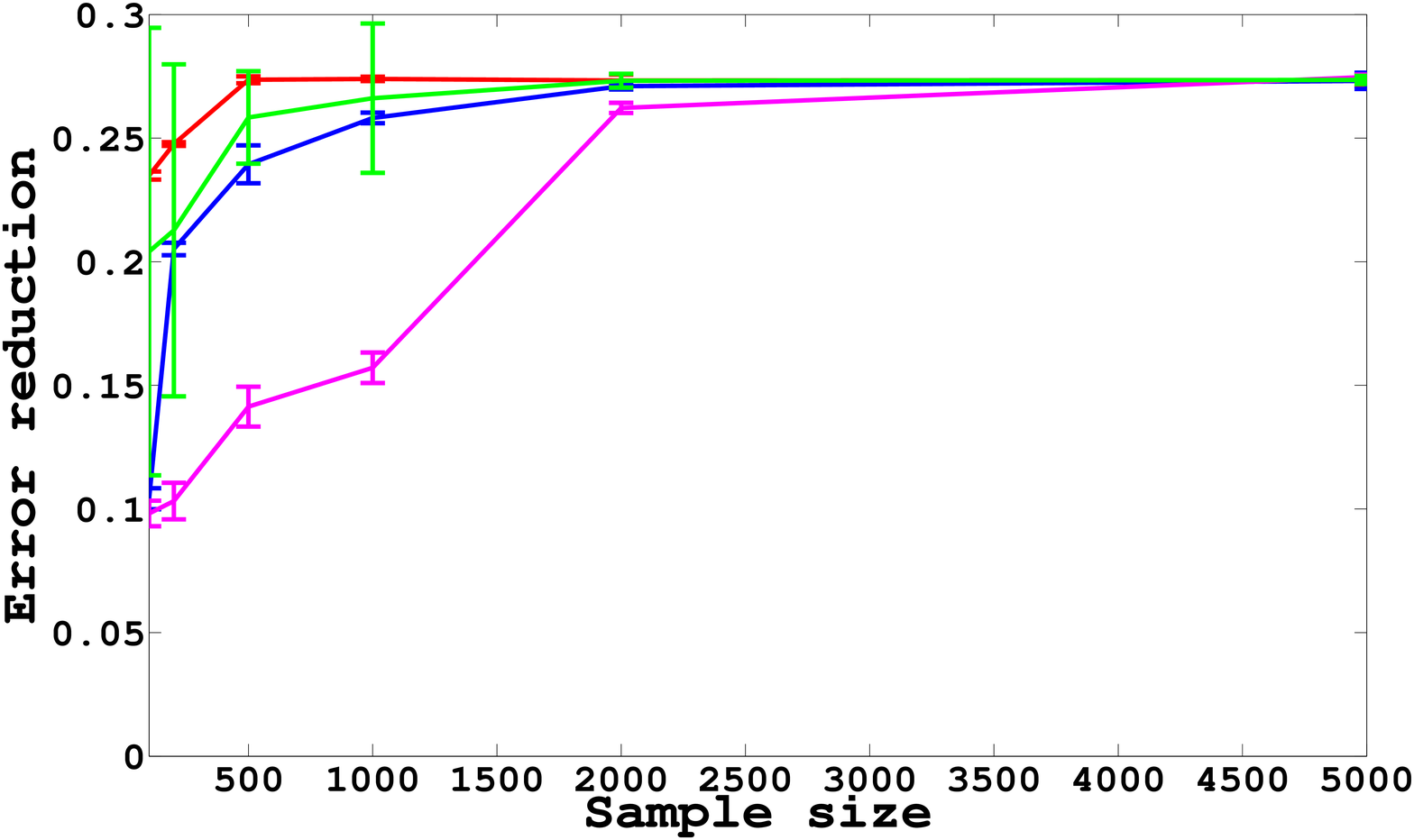}} &
\subfigure[\scriptsize NMI]{\label{fig:nmi_mnist_sampling_comparison}\includegraphics[width=4.2cm,height=3.2cm]{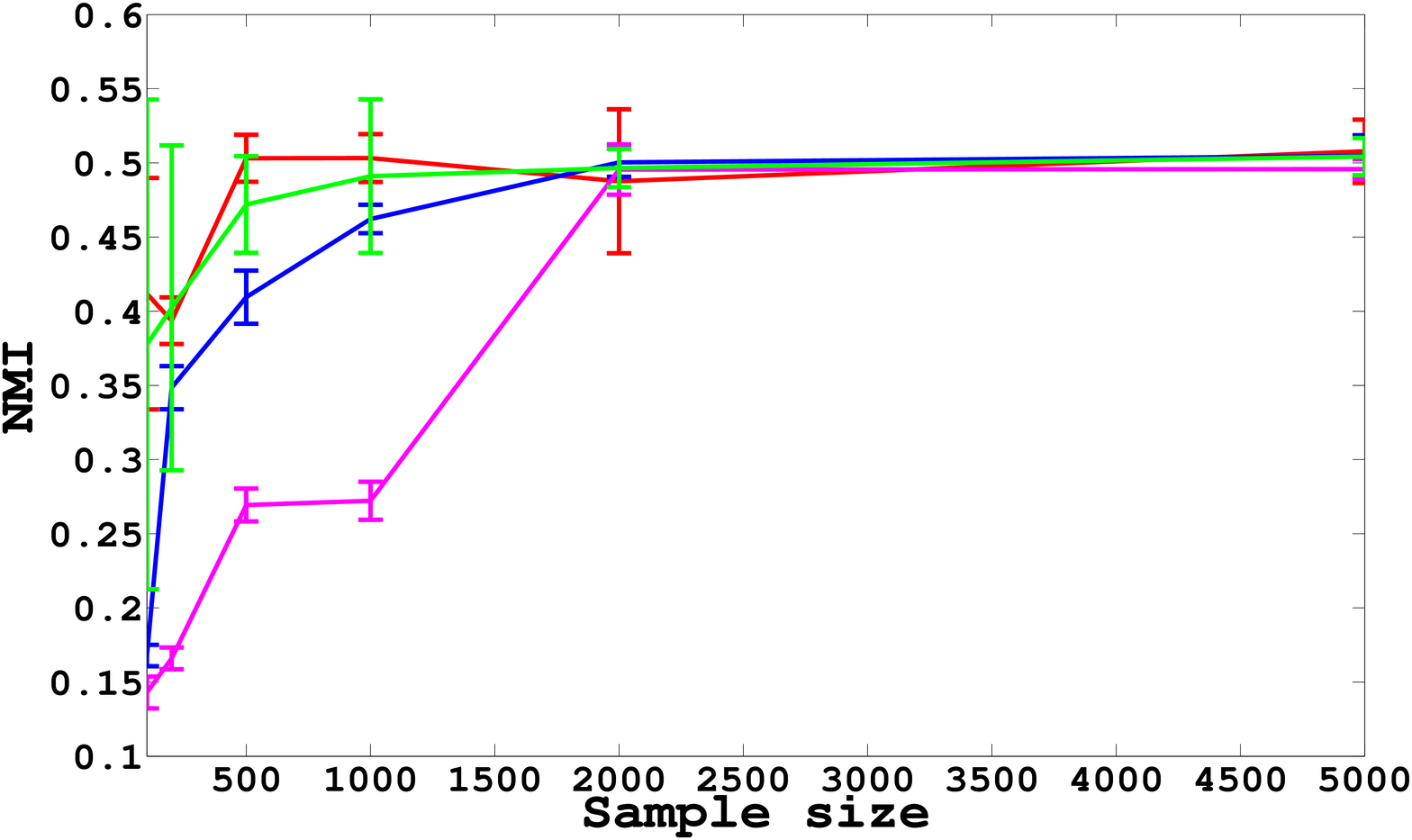}}
\end{tabular}
\caption{\footnotesize Comparison of accuracy of different sampling strategies on the Imagenet ((a) and (b)) and MNIST ((c) and (d)) data sets}
\label{fig:comp_sampling_strategies}
\end{figure*}
\subsection{Scalability}
Using the large Forest Cover Type and Network Intrusion data sets, we
demonstrate that the proposed algorithm 
is scalable to large data sets. The aKKm algorithm uses less than $40$ GB of
memory, thereby dramatically reducing the memory requirements of clustering.

\subsubsection{Datasets}
\begin{compactitem}
\item \textbf{Forest Cover Type:}  This data set~\cite{blackard1999comparative}
is composed of cartographic variables obtained from the US Geological Survey
(USGS) and the US Forest
Service (USFS) data. Each of the $581,012$ data points represents the attributes
of a $30 \times 30$
meter cell of the forest floor. There are a total of $12$ attributes, including qualitative measures
like soil type and wilderness area, and quantitative measures like slope,
elevation, and distance to hydrology. These $12$ attributes are represented using
$54$ features. The data are grouped into $7$ classes, each representing a
different forest cover type. The true cover type was determined from the USFS
Region 2 Resource Information System (RIS) data.

\item \textbf{Network Intrusion:} The Network Intrusion
data set~\cite{stolfo2000cost} contains
4,898,431 50-dimensional patterns representing TCP dump data from seven weeks
of local-area network traffic. The data are classified into 23 classes, one
class representing legitimate traffic and the remaining 22 classes
representing different types of illegitimate traffic.
\end{compactitem}

\subsubsection{Experimental setup} For these data sets, 
it is currently infeasible to
compute and store the full kernel on a single system due to memory and
computational time constraints. The aKKm algorithm alleviates this complexity issue.

We compare the performance of the aKKm algorithm on these data sets in terms of the
running time, error reduction 
and NMI, with that of the \textit{k}-means, 
the tKKm, and the nysSC algorithms. We found that the lKKm algorithm takes longer than $24$ hours to find the leaders for these large data sets, 
clearly demonstrating its non-scalability. Therefore, we eliminated this algorithm from the set of baseline algorithms.

We evaluate the efficiency of the aKKm algorithm for different sample sizes 
ranging from $100$ to $5,000$. On the Network Intrusion data set, the value of
$m$ is increased 
only up to $2,000$, as greater values of $m$ require more than $40$ GB memory.
On the Cover Type data set, we employ the RBF kernel to compute the pairwise
similarity, with the parameter $\sigma$ set to 
$0.35$. The $3$-degree polynomial kernel is 
employed for the Network Intrusion data set. The kernels and their parameters
are tuned to achieve optimal
performance. The number of clusters is again set equal to the true number of
classes in the data set.
\begin{figure*}
\centering
\begin{tabular}{cccc}
\multicolumn{4}{c}{\includegraphics[width=5cm,height=0.5cm]{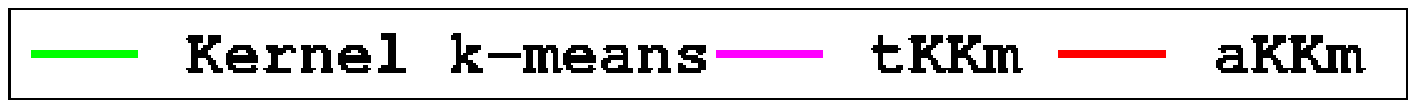}} \\
\subfigure[\scriptsize Imagenet]{\label{fig:error_red_imagenet}\includegraphics[width=4.2cm,height=3.2cm]{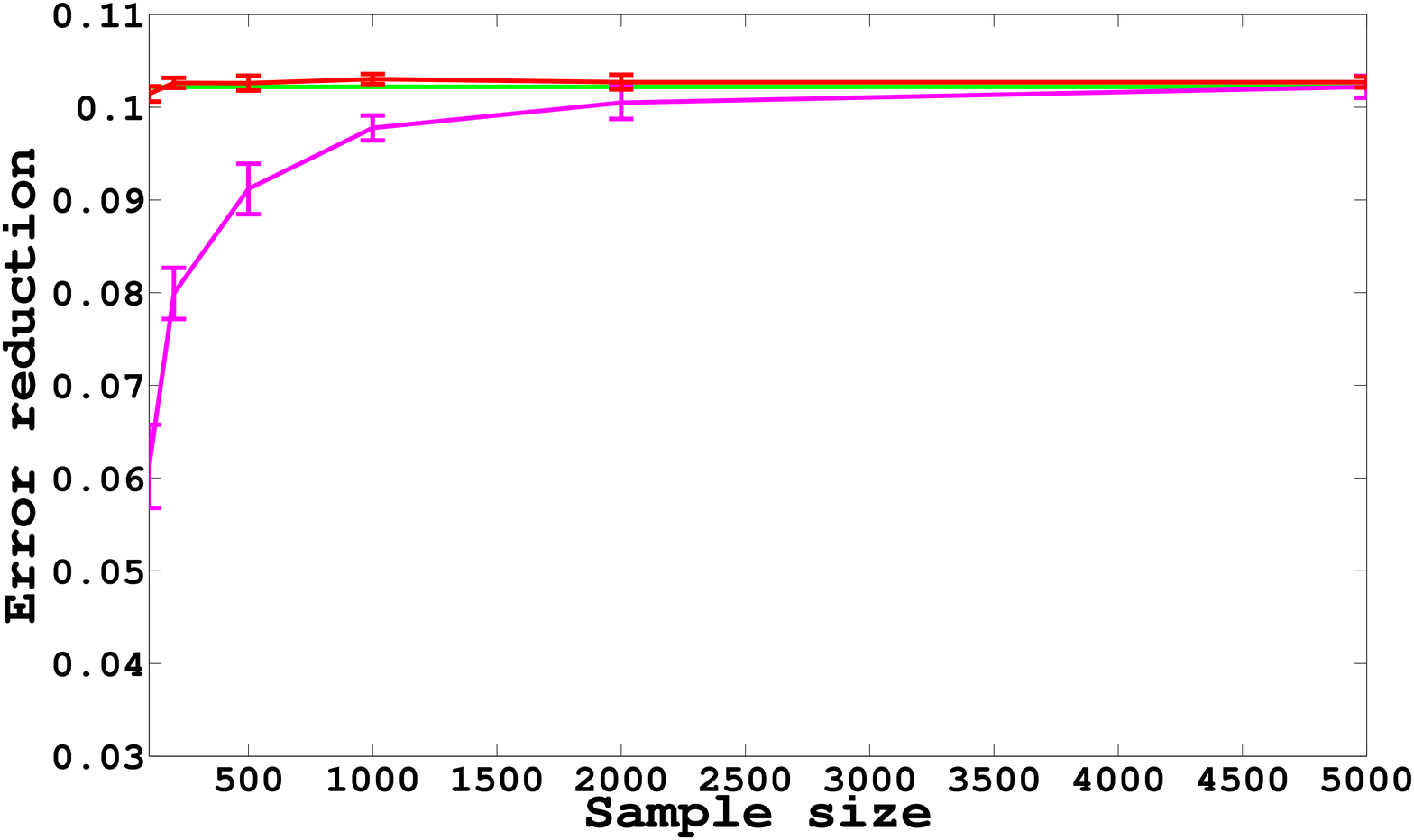}} &
\subfigure[\scriptsize MNIST]{\label{fig:error_red_mnist}\includegraphics[width=4.2cm,height=3.2cm]{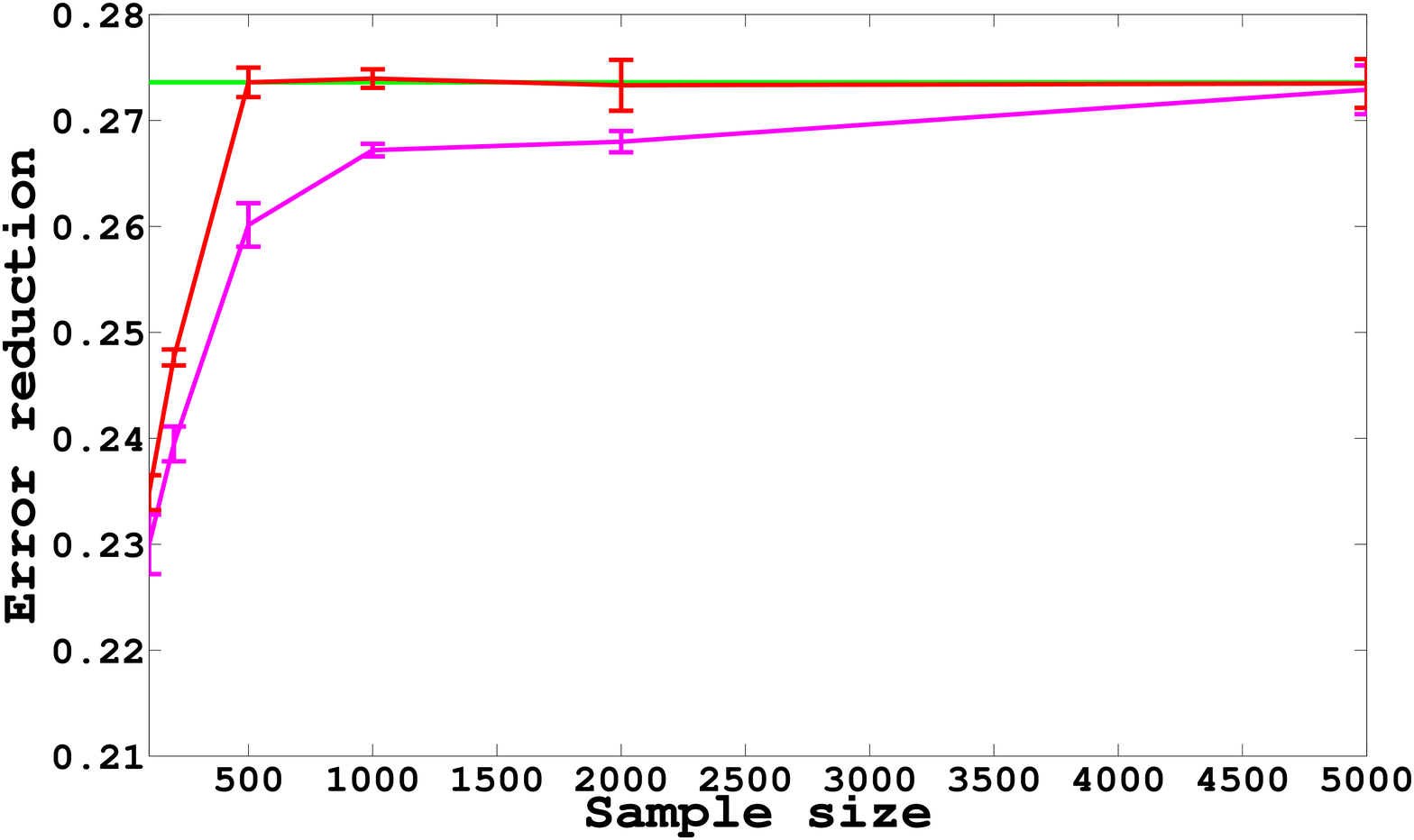}} &
\subfigure[\scriptsize Forest Cover Type]{\label{fig:error_red_covertype}\includegraphics[width=4.2cm,height=3.2cm]{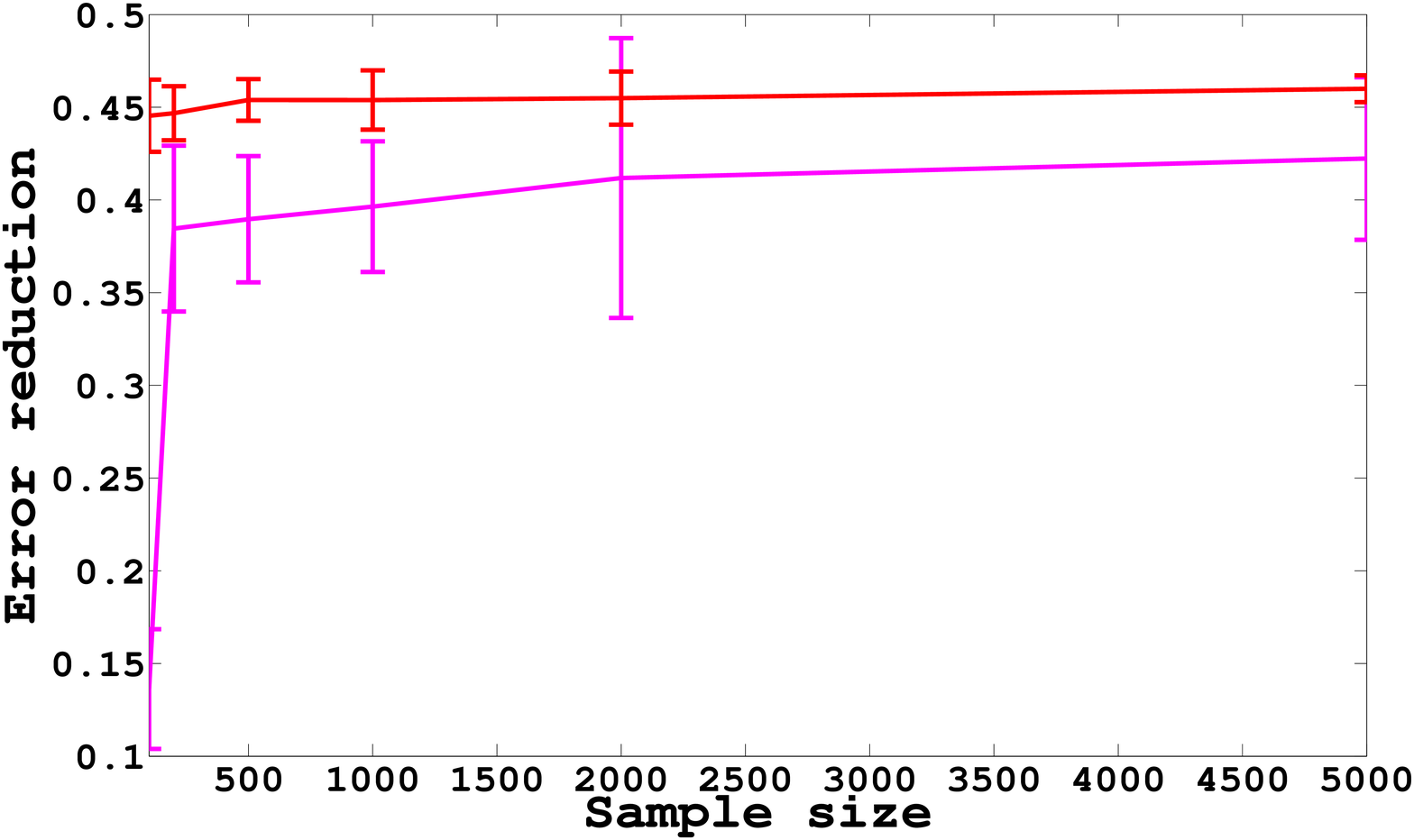}} &
\subfigure[\scriptsize Network Intrusion]{\label{fig:error_red_kddcup}\includegraphics[width=4.2cm,height=3.2cm]{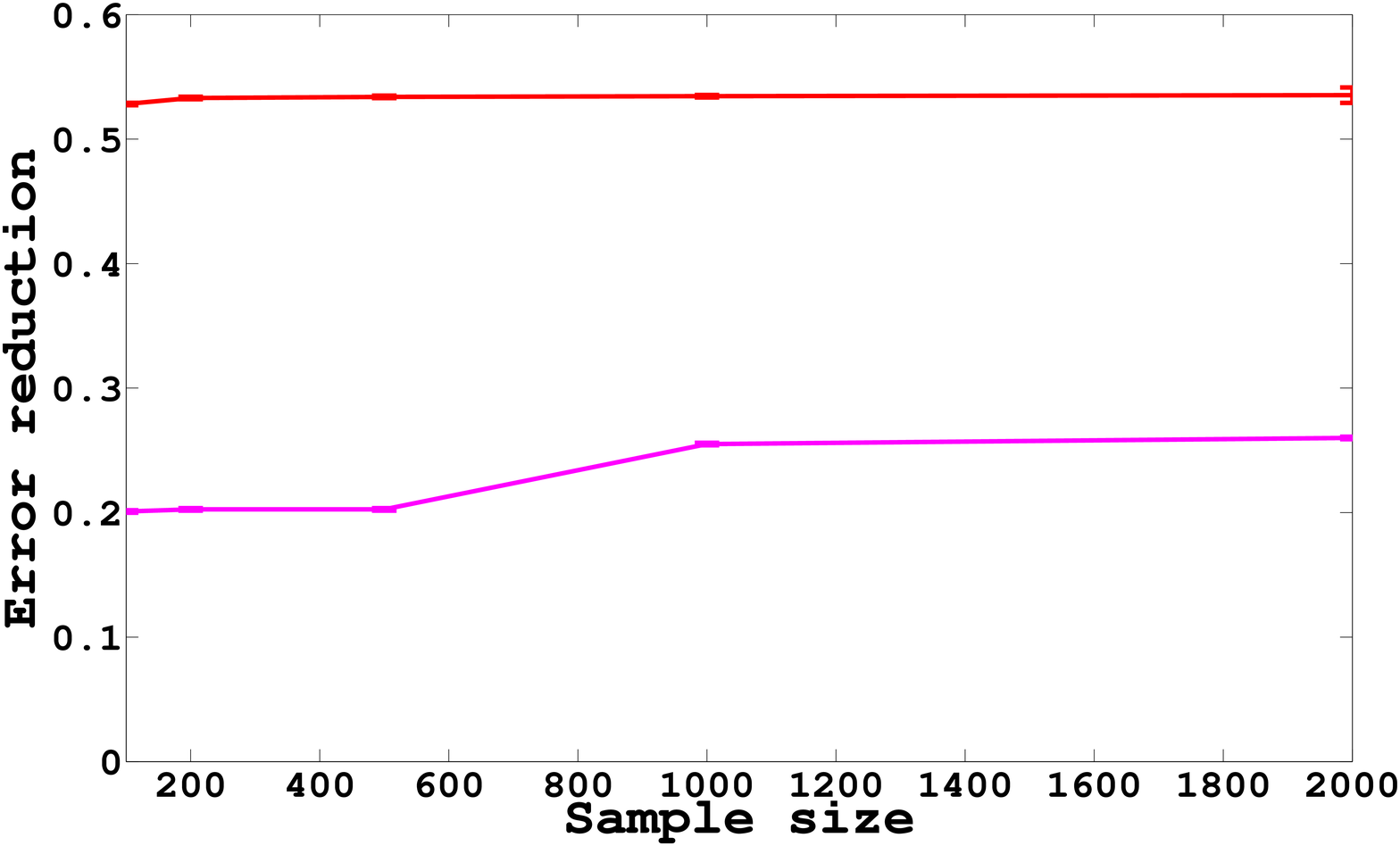}}
\end{tabular}
\caption{\footnotesize Clustering error reduction}
\label{fig_error_red}
\end{figure*}
\begin{figure*}
\centering
\begin{tabular}{cccc}
\multicolumn{4}{c}{\includegraphics[width=8cm,height=0.5cm]{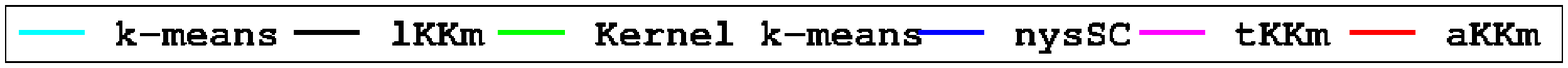}}\\
\subfigure[\scriptsize Imagenet]{\label{fig:nmi_imagenet}\includegraphics[width=4.2cm,height=3.2cm]{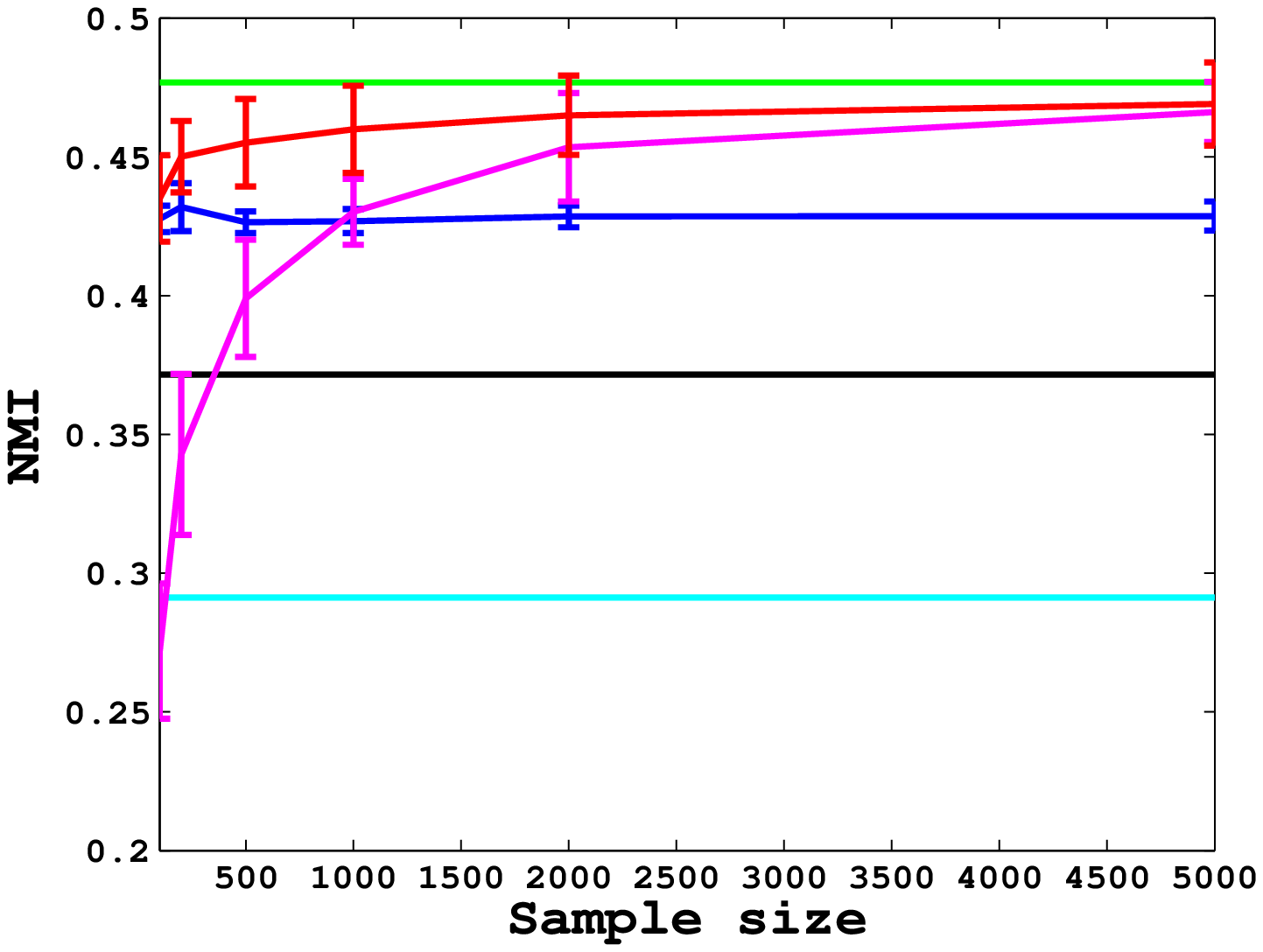}} &
\subfigure[\scriptsize MNIST]{\label{fig:nmi_mnist}\includegraphics[width=4.2cm,height=3.2cm]{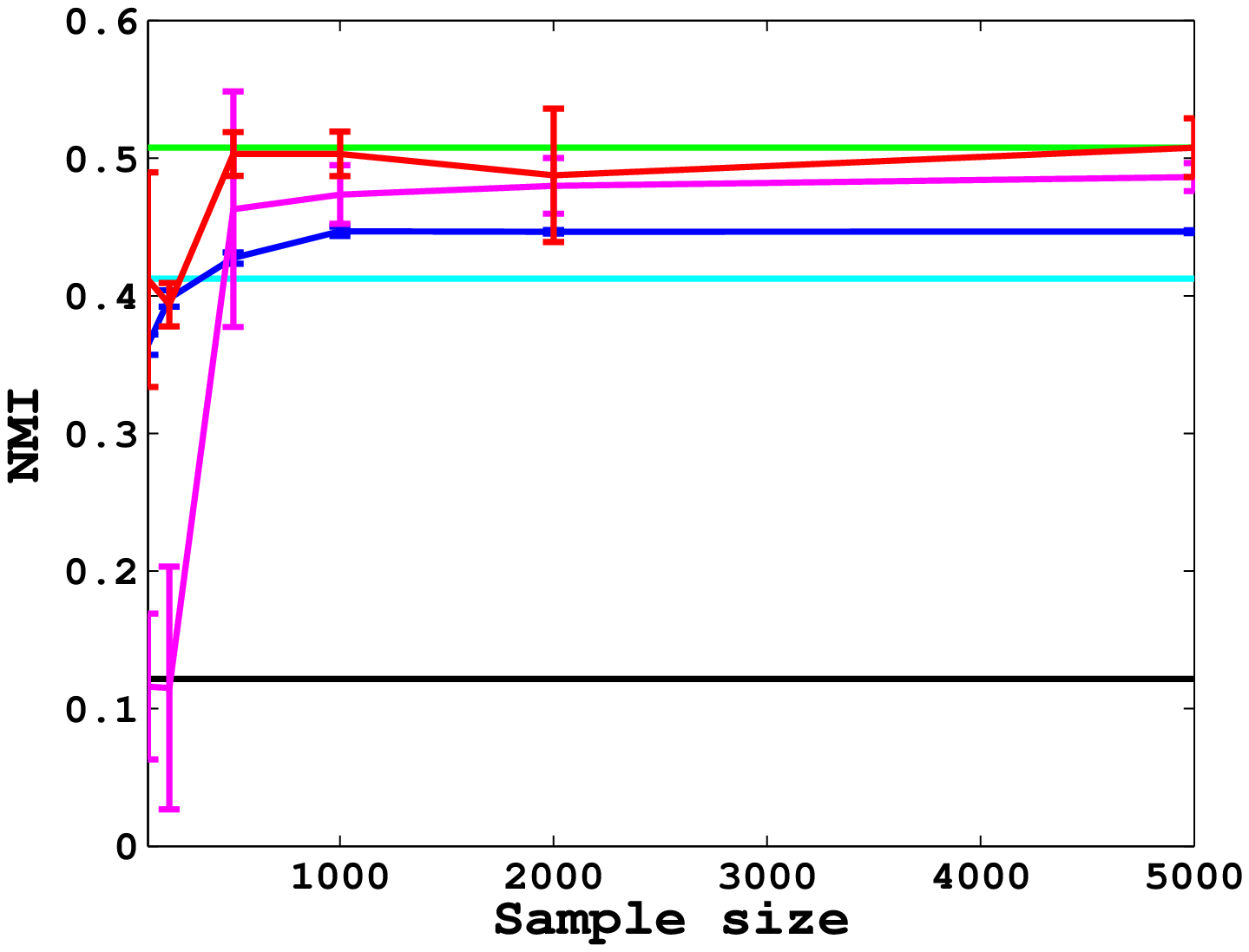}} &
\subfigure[\scriptsize Forest Cover Type]{\label{fig:nmi_covertype}\includegraphics[width=4.2cm,height=3.2cm]{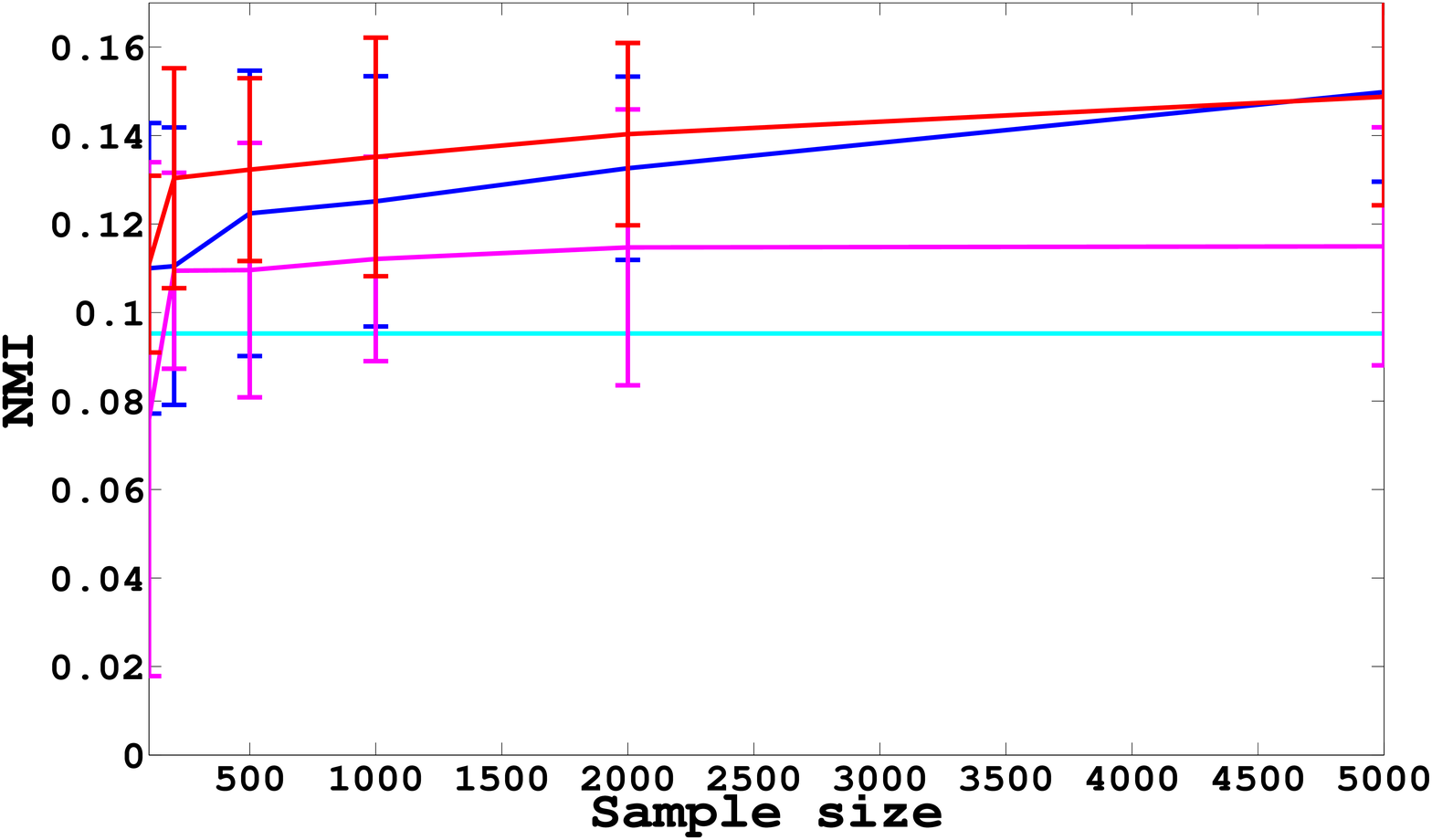}} &
\subfigure[\scriptsize Network Intrusion]{\label{fig:nmi_kddcup}\includegraphics[width=4.2cm,height=3.2cm]{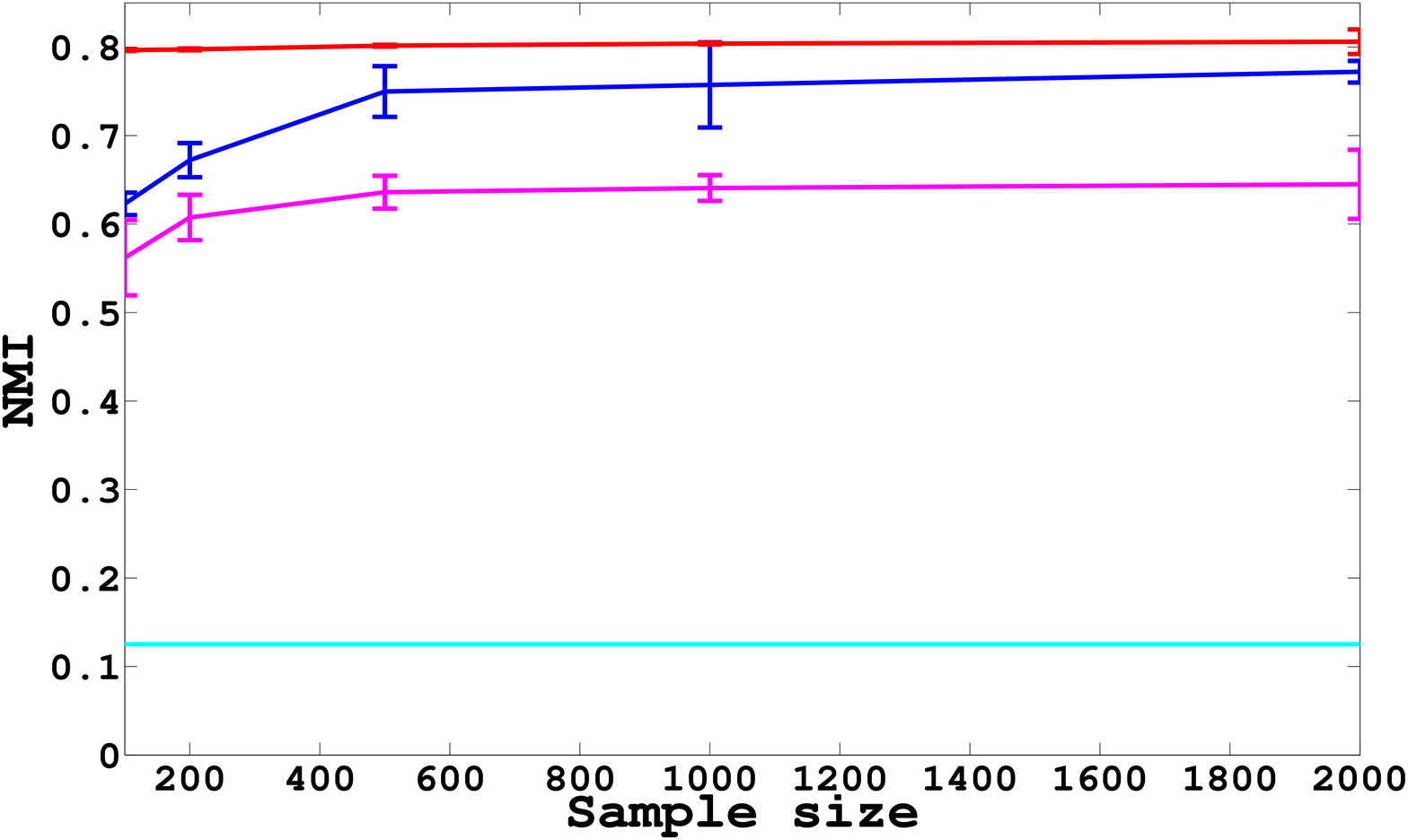}}
\end{tabular}
\caption{\footnotesize Normalized Mutual Information between the cluster labels and the true class labels}
\label{fig:nmi}
\end{figure*}

\subsubsection{Experimental results}
\textbf{Forest Cover Type:} We compare the running time of the algorithms in
Table~\ref{tbl:time_covertype}. 
As in the MNIST data set, kernel calculation time is minimal when compared to
clustering time and 
the aKKm algorithm is less efficient than the tKKm algorithm in terms of the 
clustering time. When compared to the nysSC algorithm, its
running time is higher 
when the sample size $m$ is small, but as the time taken by the nysSC 
algorithm increases cubically with $m$, the aKKm algorithm becomes more efficient as $m$ is
increased. 
It is faster than the \textit{k}-means algorithm when $m < 500$. 

Figs.~\ref{fig:error_red_covertype} and~\ref{fig:nmi_covertype} show the
effectiveness of our algorithm in terms of error reduction and NMI, respectively.
In Fig.~\ref{fig:error_red_covertype}, we observe that a much higher error
reduction is achieved by the aKKm algorithm than the tKKm algorithm even when $m=100$. Its NMI
values are also much higher than those of the 
baseline algorithms. The nysSC algorithm's NMI is similar to that
of the aKKm algorithm only when $m \geq 5,000$, when the 
nysSC algorithm is computationally more expensive.

\textbf{Network Intrusion:}  Table~\ref{tbl:time_kddcup} shows the running time
of all the algorithms.
As observed in the results for the earlier data sets, aKKm takes longer
than tKKm but is much faster than the nysSC algorithm. Fig.~\ref{fig:error_red_kddcup} shows
that aKKm achieves a better error 
reduction than tKKm. It also performs better than both the
tKKm and nysSC algorithms 
in terms of NMI, as shown in Fig.~\ref{fig:nmi_kddcup}.

These results demonstrate that the aKKm algorithm is an efficient technique for
clustering large data sets. 

\begin{table}
\centering
\scriptsize
\begin{tabular}{|c|c|}
\hline
\textbf{Data set}&\textbf{Running time (seconds)}\\
\hline
Imagenet&	0.71 ($\pm0.0.05$) \\
\hline
MNIST&	0.84 ($\pm0.02$) \\
\hline
Forest Cover Type &4.14 ($\pm0.02$) \\
\hline
Network Intrusion&	114.87 ($\pm1.30$)\\
\hline
\end{tabular}
\caption{\footnotesize Time taken by MCLA to combine $10$ partitions}
\label{tbl:time_ensemble}
\end{table}

\begin{figure*}
\centering
\begin{tabular}{cccc}
\multicolumn{4}{c}{\includegraphics[width=4cm,height=0.5cm]{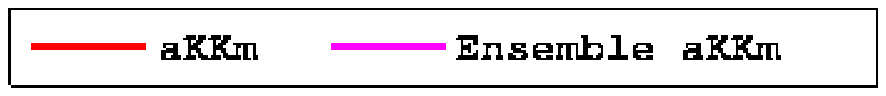}} \\
\subfigure[\scriptsize Imagenet]{\label{fig:nmi_imagenet_ensemble}\includegraphics[width=4.2cm,height=3.2cm]{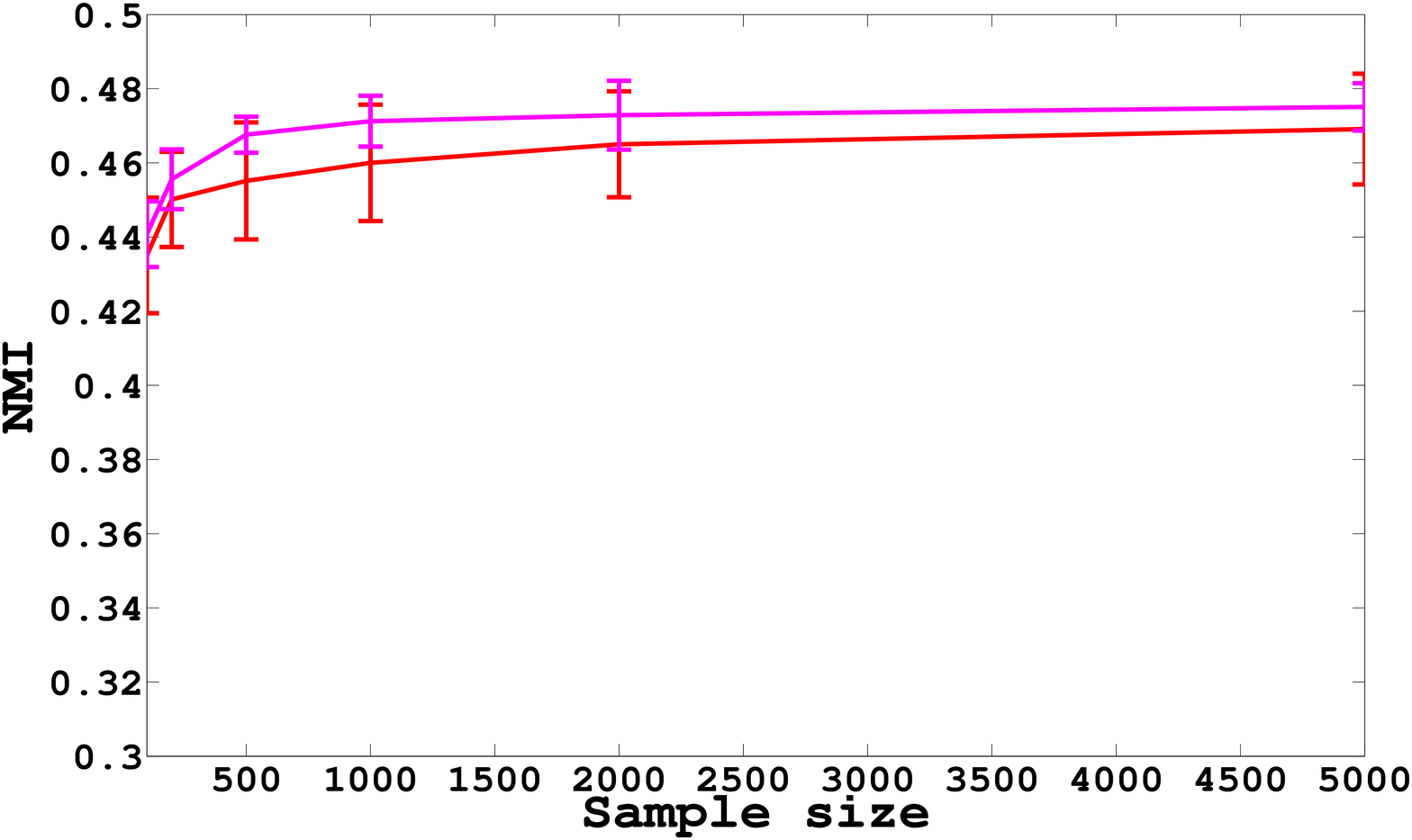}} &
\subfigure[\scriptsize MNIST]{\label{fig:nmi_mnist_ensemble}\includegraphics[width=4.2cm,height=3.2cm]{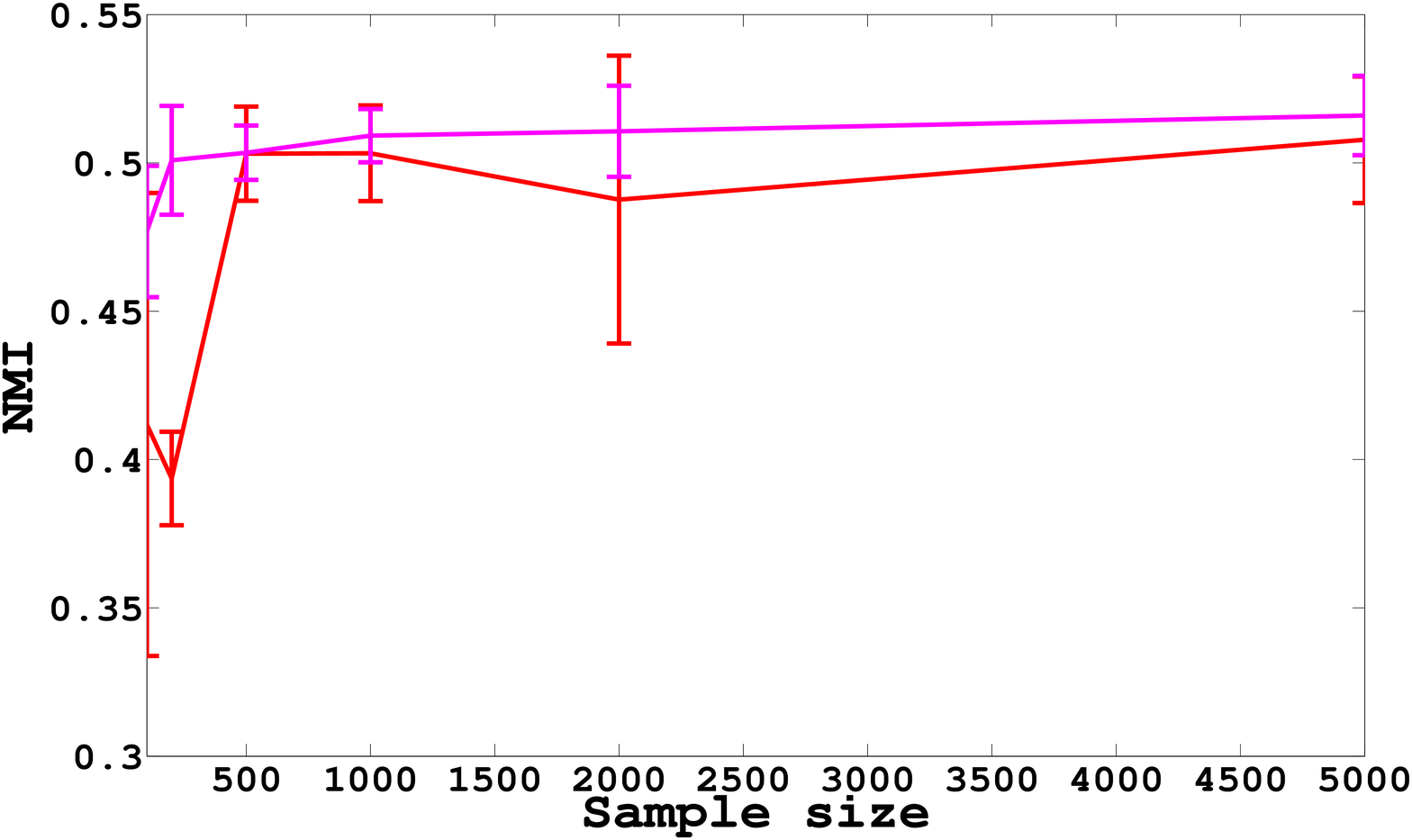}} &
\subfigure[\scriptsize Forest Cover Type]{\label{fig:nmi_covertype_ensemble}\includegraphics[width=4.2cm,height=3.2cm]{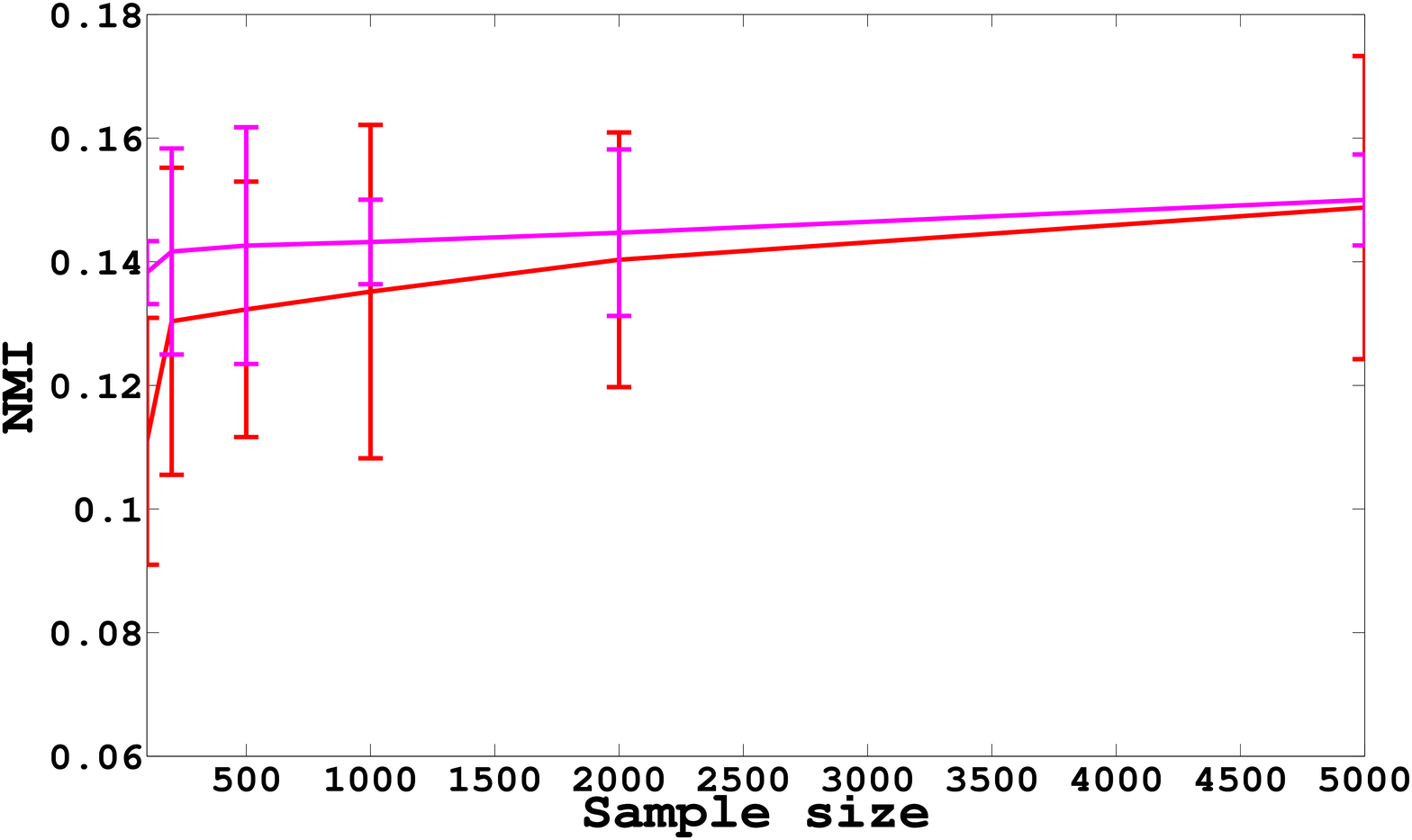}} &
\subfigure[\scriptsize Network Intrusion]{\label{fig:nmi_kddcup_ensemble}\includegraphics[width=4.2cm,height=3.2cm]{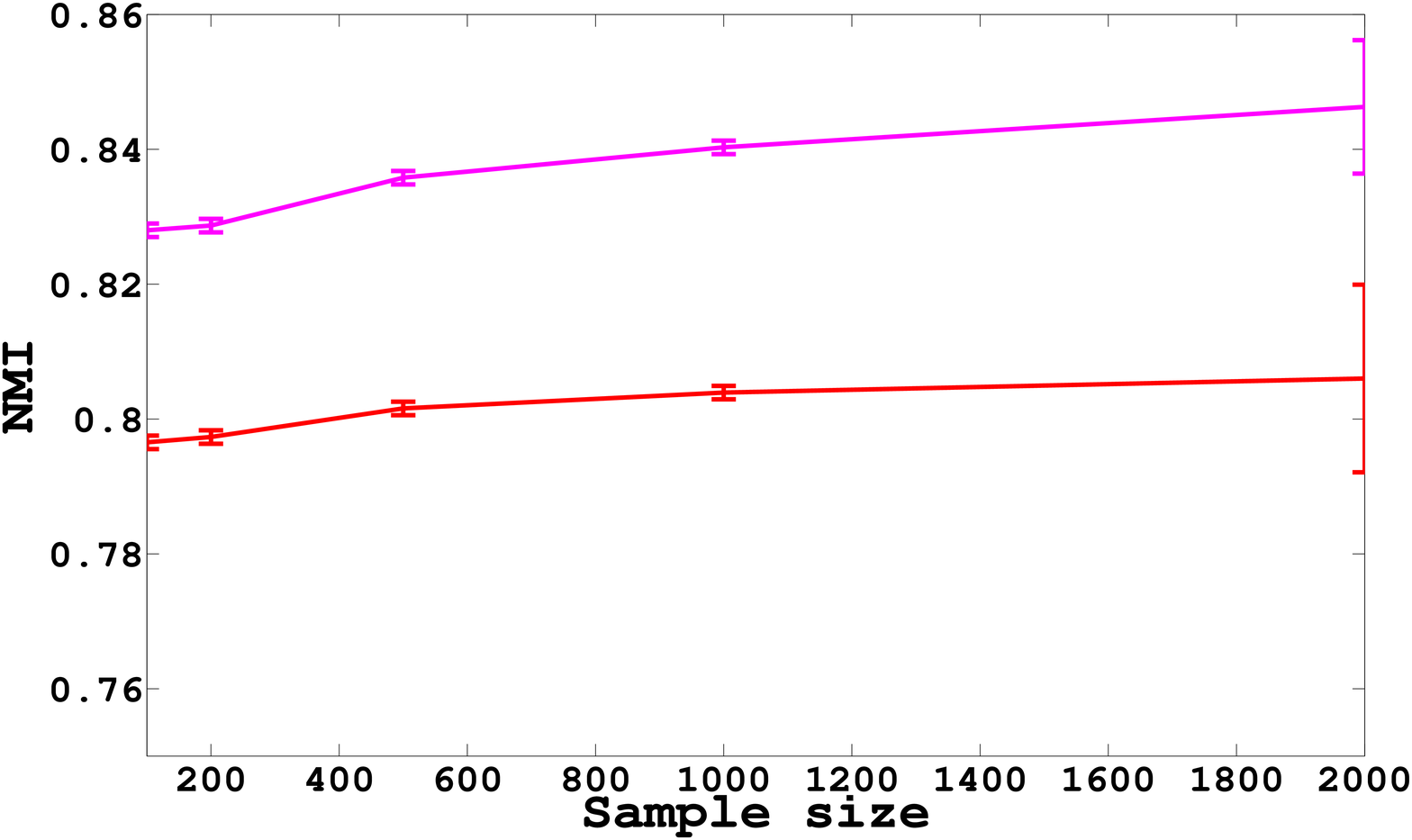}}
\end{tabular}
\caption{\footnotesize Normalized Mutual Information of the consensus partition labels with respect to the true class labels}
\label{fig:nmi_ensemble}
\end{figure*}

\subsection{Ensemble aKKm} We employ the ensemble aKKm algorithm to combine $10$ ensemble partitions. 
The times taken for combining the partitions
(averaged over $10$ runs), for each of the data sets, are shown in Table~\ref{tbl:time_ensemble}. We find that
these
times are small when compared to the clustering times and hence, do not
significantly impact the overall running time. Fig.~\ref{fig:nmi_ensemble} shows
the
improvement in NMI achieved through the use of ensembles on the four data sets. 
A significant improvement is observed, especially when the sample size $m$ is
small. For example, in Fig.~\ref{fig:nmi_mnist_ensemble}, an NMI of around $0.48$ is obtained 
on the MNIST data set, even for a small sample size of $m=100$. 
The NMI increases by about $15\%$ and becomes 
almost equal to the average NMI obtained for a sample size $m=1,000$. On the
Cover Type and Network 
Intrusion data sets, there are significant improvements in the NMI values. We
obtain a 
good clustering accuracy for small values of $m$, thereby
enhancing the efficiency of the approximate kernel \textit{k}-means algorithm.  

\section{Conclusions}
We have proposed an efficient approximation for the kernel \textit{k}-means
algorithm,
suitable for large data sets. The key idea is to avoid computing the full kernel
matrix by restricting the cluster centers to a subspace spanned by a small set
of randomly sampled data points. We show theoretically and empirically that the
proposed algorithm is efficient in terms of both computational complexity and
memory requirement, and is able to yield similar clustering results as the
kernel
\textit{k}-means algorithm using the full kernel matrix. In most cases, the
performance
of our algorithm is better than that of other popular large scale kernel
clustering algorithms. By integrating ensemble clustering methods with the proposed algorithm, its efficiency is further enhanced. 
In the future, we plan to further enhance the scalability of kernel clustering 
by devising more efficient kernel approximation techniques.
We also plan to extend these ideas to semi-supervised clustering.

\section*{Acknowledgements}
This research was supported by the Office of Naval Research (ONR Grant
N00014-11-1-0100). Havens is supported by the National Science Foundation 
under Grant \#1019343 to the Computing Research Association for the CI 
Fellows Project.
\begin{appendix}
\textbf{Proof of Theorem~\ref{thm:m_est}}\newline
To prove this theorem, we use the following results from~\cite{gittens2011spectral},~\cite{candes2007sparsity}, and~\cite{smale-2009-geometry}:
\begin{lemma}(Theorem $1$ from~\cite{gittens2011spectral})
\label{lem:gittens}
Let $A$ be a $n \times n$ positive semi-definite matrix and $S \in \{0,1\}^{n \times m}$ be a random sampling matrix. Let $A$ be partitioned as
\begin{equation*}
 A = \left[\begin{array}{cc}Z_1 & Z_2\end{array}\right]\left[\begin{array}{cc}\Sigma_1 & 0 \\0 & \Sigma_2\end{array}\right]
\left[
\begin{array}{c}
 Z_1^{\top}  \\
Z_2^{\top}
\end{array}
\right],
\end{equation*}
where $Z_1 \in \Re^{n \times C}$, $Z_2 \in \Re^{n \times (n-C)}$, $\Sigma_1 \in \Re^{C \times C}$ and $\Sigma_2 \in \Re^{n \times (n-C)}$.
Define $\Omega_1 = Z_1^{\top}S$ and $\Omega_2 = Z_2^{\top}S$. Assume $\Omega_1$ has full row rank. Then the spectral approximation error of the Nystrom extension 
of $A$ using $S$ as the column sampling matrix satisfies
\begin{eqnarray*}
 \left \| A - AS(S^{\top}AS)^{\dag}S^{\top}A\right\| &\leq& \left \| \Sigma_2\right \|_2 \left (1 + \left \|\Omega_2\Omega_1^{\dag}\right\|_2^2\right).
\end{eqnarray*}
\end{lemma}
\begin{lemma}(Theorem $1.2$ from~\cite{candes2007sparsity})
\label{lem:candes}
Let $Z \in \Re^{n \times n}$ include the eigenvectors of a positive semi-definite $A$ with coherence $\tau$, where coherence is defined in~\eqref{eqn:coherence}. 
Let $Z_1 \in \Re^{n \times C}$ represent the first $C$ columns of $Z$, containing the first $C$ eigenvectors of $A$, and $S \in \{0,1\}^{n \times m}$ represent the first 
$m$ columns of a random permutation matrix of size $n$. We have, with probability atleast $1-\delta$, 
\begin{eqnarray*}
\left \| \frac{1}{m}Z_1^{\top}SS^{\top}Z_1 - I\right \|_2 < \frac{1}{2},
\end{eqnarray*}
provided that $m \geq C\tau\max(C_1\ln k, C_2 \ln(3/\delta))$, for some fixed positive constants $C_1$ and $C_2$. 
\end{lemma}
\begin{lemma}(Lemma $2$ from~\cite{smale-2009-geometry})
\label{lem:smale}
Let $\mathcal{H}$ be a Hilbert space and $\xi$ be a random variable on $(Z,\rho)$ with values in $H$. Assume $\|\xi\| \leq M < \infty$ almost surely. Denote
$\sigma^2(\xi) = \E(\|\xi\|^2)$. Let $\{z_i\}_{i=1}^m$ be independent random
drawers of $\rho$. For any $0 < \delta < 1$, with confidence $1 - \delta$,
\begin{equation*}
\label{eqn:smale_inequality}
    \left\|\frac{1}{m}\sum_{i=1}^m (\xi_i - \E[\xi_i]) \right\| \leq
\frac{2M\ln(2/\delta)}{m} + \sqrt{\frac{2\sigma^2(\xi)\ln(2/\delta)}{m}}.
\end{equation*}
\end{lemma}
\begin{proof}
Let $\mathbf{a_i}$ and $\mathbf{b_i}$ represent the $i^{th}$ rows of $Z_1$ and $Z_2$, respectively. Let $\Delta$ be the subset of rows of $Z_1$ and $Z_2$ selected by $S$. 
Using Lemma~\ref{lem:smale}, we have 
\begin{eqnarray}
\label{eqn:limit1}
\left \|Z_2^{\top}SS^{\top}Z_1\right \|_2 &=& \left \|\sum \limits_{k \in \Delta}\left (b_ka_k^{\top} - \E[b_ka_k^{\top}]\right )\right \|_2 \\ \nonumber
&\leq& 2M \ln(2/\delta) + \sqrt{2m\sigma^2\ln(2/\delta)},
\end{eqnarray}
where $M = \max \limits_j \left \| b_j a_j^{\top}\right \|_2 \leq \max \limits_j 
\sqrt{\left | b_j \right |^2 \left | a_j \right |^2} \leq \tau\sqrt{C/n}$ and $\sigma^2 = \E\left [\left \| b_j a_j^{\top}\right \|_2\right] \leq \tau C/n$.
\newline Substituting $\Omega_2\Omega_1^{\dag} = Z_2^{\top}SS^{\top}Z_1(Z_1^{\top}SS^{\top}Z_1)^{-1}$ in the result of Lemma~\ref{lem:gittens}, we have with probability $1-2\delta$,
\begin{eqnarray}
\label{eqn:limit2}
&& \left \| A - AS(S^{\top}AS)^{\dag}S^{\top}A\right \| \\ \nonumber
&\leq& \lambda_{C+1}\left (1 + \left \|Z_2^{\top}SS^{\top}Z_1\right \|_2^2 \left\|(Z_1^{\top}SS^{\top}Z_1)^{-1}\right\|_2^2\right)
\end{eqnarray}
We obtain the result in the theorem by combining the result of Lemma~\ref{lem:candes} with equations~\eqref{eqn:limit1} and~\eqref{eqn:limit2}, and substituting $A$, $AS$ and $S^{\top}AS$ 
with $K$, $K_B$ and $\Kh$ respectively.
\end{proof}
\end{appendix}

\bibliographystyle{IEEEtran}
\bibliography{draft_arxiv}
\end{document}